\documentclass{article}

\usepackage{microtype}
\usepackage{graphicx}
\usepackage{booktabs}
\usepackage{bbold}
\usepackage{multirow}
\usepackage{geometry}
\usepackage{natbib}
\usepackage{parskip}
\usepackage{algorithm}
\usepackage{algorithmic}
\usepackage{authblk}
\usepackage{subfig}

\usepackage[colorlinks,citecolor=blue,linkcolor=blue,urlcolor=blue]{hyperref}

\usepackage{amsmath}
\usepackage{amssymb}
\usepackage{mathtools}
\usepackage{amsthm}

\usepackage[capitalize,noabbrev]{cleveref}

\theoremstyle{plain}
\newtheorem{theorem}{Theorem}[section]
\newtheorem{proposition}[theorem]{Proposition}
\newtheorem{lemma}[theorem]{Lemma}
\newtheorem{corollary}[theorem]{Corollary}
\theoremstyle{definition}
\newtheorem{definition}[theorem]{Definition}

\theoremstyle{remark}
\newtheorem{remark}[theorem]{Remark}

\title{E-Values Expand the Scope of Conformal Prediction}

\author[1]{Etienne Gauthier\thanks{\texttt{etienne.gauthier@inria.fr}}}
\author[1]{Francis Bach}
\author[1,2]{Michael I. Jordan}

\affil[1]{Inria, Ecole Normale Supérieure, PSL Research University}
\affil[2]{Departments of EECS and Statistics, University of California, Berkeley}

\begin{document}

\maketitle

\begin{abstract}
Conformal prediction is a powerful framework for distribution-free uncertainty quantification. The standard approach to conformal prediction relies on comparing the ranks of prediction scores: under exchangeability, the rank of a future test point cannot be too extreme relative to a calibration set. This rank-based method can be reformulated in terms of p-values. In this paper, we explore an alternative approach based on e-values, known as conformal e-prediction. E-values offer key advantages that cannot be achieved with p-values, enabling new theoretical and practical capabilities. In particular, we present three applications that leverage the unique strengths of e-values: batch anytime-valid conformal prediction, fixed-size conformal sets with data-dependent coverage, and conformal prediction under ambiguous ground truth. Overall, these examples demonstrate that e-value-based constructions provide a flexible expansion of the toolbox of conformal prediction.
\end{abstract}

\section{Introduction}

Conformal prediction \citep{vovk2005algorithmiclearning} is a widely used statistical framework that provides predictive models with valid uncertainty quantification. The core objective is to construct conformal sets that, given features $X \in \mathcal{X}$, contain the true target $Y \in \mathcal{Y}$ with high probability. 

These guarantees are distribution-free, meaning they hold under any underlying data-generating distribution $\mathbb{P}$, provided the data are exchangeable. Consider a dataset of $n$ exchangeable data points $\{(X_i,Y_i) \}_{i=1,...,n}$, where each pair $(X_i,Y_i)$ is drawn from some distribution $\mathbb{P} = \mathbb{P}_X \otimes \mathbb{P}_{Y | X}$ over $\mathcal{X} \times \mathcal{Y}$. This dataset, referred to as the calibration set, is used to construct conformal sets. Let $\alpha \in (0, 1)$ be a given error level. Given a new pair $(X_{n+1}, Y_{n+1}) \sim \mathbb{P}$ such that $(X_1,Y_1),...,(X_n,Y_n),(X_{n+1},Y_{n+1})$ remain exchangeable, the objective is to construct a conformal set $\hat{C}_n(X_{n+1})$ satisfying
\begin{equation}
\label{cp_goal}
\mathbb{P}(Y_{n+1} \in \hat{C}_n(X_{n+1})) \geq 1 - \alpha,
\end{equation}
where the probability is taken over both the calibration set $\{(X_i,Y_i) \}_{i=1,...,n}$ and the new data point $(X_{n+1},Y_{n+1})$.

In conformal prediction, uncertainty quantification relies on a score function $S: \mathcal{X}\times\mathcal{Y}\rightarrow\mathbb{R}$, which is typically derived from a pre-trained model $f$. The choice of $S$ depends on the specific problem setting. In this work, we consider negatively-oriented scores, where smaller values indicate better performance, and assume these scores are positive.\footnote{The assumption that scores are negatively-oriented is not restrictive, as any positively-oriented score can be transformed, for example, via $S_{\text{negative}} = 1/(S_{\text{positive}} + \varepsilon)$
with $\varepsilon$ preventing division by zero. Similarly, most scores used in conformal prediction are already nonnegative. We explicitly assume nonnegativity to simplify the construction of e-variables in conformal e-prediction, as in (\ref{def_e_var}).} An example of such a score is the cross-entropy loss used in classification tasks:
\begin{equation}
\label{cross-entropy}
    S(X,Y) = - \log p_f(Y|X),
\end{equation}
where $p_f(Y|X)$ is the predicted probability assigned to the true class $Y$.

In standard conformal prediction methods, the idea is to rank nonconformity scores $S_i := S(X_i,Y_i)$ and compare their relative positions. This approach relies on the exchangeability of the scores, ensuring that the rank of a new test point is unlikely to be overly extreme compared to those in the calibration set. For completeness, we include an overview of this classical framework in Appendix \ref{app:intro-conformal-prediction}. 

Despite its theoretical appeal and simplicity, traditional conformal prediction faces significant limitations in practice. Some of these have to do with the lack of conditional guarantees, and there has been significant recent work aiming to address that limitation~\citep{Jung, gibbs2024conformalpredictionconditionalguarantees}.  In the current paper, we highlight three additional scenarios where conventional methods fall short: batch anytime-valid conformal prediction, fixed-size conformal sets with data-dependent coverage, and conformal prediction under ambiguous ground truth. As we emphasize in the next three subsections, these scenarios all reflect real-world problems where one would like to make use of conformal methods.  While conventional conformal prediction based on p-values falls short of addressing these problems, we will show that they are all solved by using e-values to construct conformal sets instead of p-values, a methodology known as \emph{conformal e-prediction}.  Our overall suggestion is that conformal e-prediction provides a flexible tool relative to classical conformal prediction that enhances the applicability of conformal methods in the complex, diverse, and dynamical settings of modern machine learning. The code implementation for our experiments is available at \url{https://github.com/GauthierE/evalues-expand-cp}.

\subsection{Batch anytime-valid conformal prediction}

In Section \ref{sec:ville}, we consider the following problem: data arrive sequentially in batches $b_t$, $t =1,2,\dotsc,$ with each batch containing $n_t \ge 0$ calibration data points $S_1^t,...,S_{n_t}^t$ and a test data point $S_{n_t+1}^t$. Importantly, the batches arrive sequentially, one after the other, and the number of batches is possibly unknown in advance. Between batches, the data distribution may shift, meaning the data from one batch is not identically distributed with the next. We simply assume that the data within each batch, including both calibration and test data points, are exchangeable, conditional on the previous data batches.

Given $\alpha \in (0,1)$, the goal is to construct a sequence of batch anytime-valid conformal sets $\hat{C}_t$ for all $t \ge 1$,  based on all past batches and the calibration data of batch $t$, satisfying:
\begin{equation}
\label{goal}
    \mathbb{P}(\forall t \ge 1, \ S_{n_t+1}^t \in \hat{C}_t) \ge 1-\alpha,
\end{equation}
where the probability is taken over all data points. If (\ref{goal}) holds, we call $\{\hat{C}_t \}_{t \ge 1}$ a sequence of batch anytime-valid conformal sets. As we now discuss, this problem arises in many real-life scenarios where data are received sequentially. 

\subsubsection{Motivating example 1: Pharmaceutical drug deployment across hospitals}

A real-world example of this problem arises in pharmaceutical drug deployment across hospitals. A company introduces a new drug, and a regulatory agency seeks statistical guarantees on its efficacy. The drug’s effect on patient $i$ in hospital $b_t$ is denoted by $S_i^t$, where each hospital corresponds to a data batch. After an initial calibration phase in hospital $b_t$, where the drug is tested on $n_t$ volunteer patients, the goal is to construct a conformal prediction set for $S_{n_t+1}^t$, representing a future patient in that hospital.

The deployment is inherently sequential: hospitals test the drug at different times, and results arrive progressively. The agency must provide anytime-valid statistical guarantees for all hospitals, constructing conformal prediction sets for future patients $S_{n_t+1}^t$ across all sites, regardless of how many batches $t$ have been received. 

Moreover, data distributions vary across hospitals due to differences in patient demographics, lifestyles, and medical histories. For instance, urban and rural hospitals may observe different drug efficacy patterns. The regulatory agency must account for these distributional shifts while leveraging prior batches to enhance predictive reliability.

\subsubsection{Motivating example 2: Quality control in supply chains}

In a production supply chain, a company regularly receives shipments of parts from a supplier. Each batch contains $q_t$ units, each with a quality rating $S_i^t$, assessed by the company. The batch size $q_t$ varies based on demand.

While unit quality within a batch is typically exchangeable, it may fluctuate across batches due to changes in the supplier’s production process (e.g., machinery wear) or transportation conditions. To manage this uncertainty, the company can apply batch anytime-valid conformal prediction to obtain quality guarantees. Specifically, for each batch, it selects $n_t<q_t$ units to measure quality, forming a calibration set $S_1^t,...,S_{n_t}^t$. Then, it predicts the quality of a randomly chosen product $S_{n_t+1}^t$ from the remaining units. Since quality assessment is costly, $n_t$ is typically small relative to $q_t$.

The key advantage of this method is its validity for any batch and any stopping time~$\tau$. The company can flexibly decide when to stop quality assessments, depending on prior observations. This adaptive approach helps optimize testing costs while ensuring reliable predictions, improving supply chain efficiency. 

\subsubsection{Limitations of traditional conformal prediction methods}

The standard method in conformal prediction produces a conformal set $\hat{C}_t$ for a batch of exchangeable data $S_1^t,...,S_{n_t}^t,S_{n_t+1}^t$ such that
\[
1- \alpha \le \mathbb{P}(S_{n_t+1}^t \in \hat{C}_t) < 1-\alpha + \frac{1}{n_t+1},
\]
where the rightmost inequality holds if the data are almost surely distinct. In dynamic scenarios such as batch anytime-valid conformal prediction, simply applying existing conformal prediction methods to each batch does not yield batch anytime-valid guarantees. We formalize this fact in the following lemma:

\begin{lemma}
Let $\hat{C}_t$ be the conformal set obtained using standard conformal prediction on batch $b_t$ for any $t \ge 1$. Suppose that: (i) almost surely, there are no ties within any batch; (ii) batches are independent; and (iii) $n_t \ge \frac{2}{\alpha}-1$ for all $t \ge 1$.

Then, $\{\hat{C}_t \}$ is not a sequence of batch anytime-valid conformal sets.
\end{lemma}
\begin{proof}
    Let $T := \Big\lceil \frac{\log(1-\alpha)}{\log(1-\alpha/2)} \Big\rceil$. We have:
    \begin{align*}
    \mathbb{P}(\forall t \ge 1, \ S_{n_t+1}^t \in \hat{C}_t) &\le \mathbb{P}(\forall t =1,...,T, \ S_{n_t+1}^t \in \hat{C}_t)\\
    &= \prod_{t=1}^T \mathbb{P}(S_{n_t+1}^t \in \hat{C}_t)\\
    &< \prod_{t=1}^T \left(1-\alpha+\frac{1}{n_t+1}\right)\\
    &< (1 - \alpha/2)^T \le 1-\alpha.
    \end{align*}
    The equality comes from the fact that all sets $\hat{C}_t$ depend only on data from batch $b_t$ and batches are independent. The last two inequalities come from the assumptions on $n_t$ and the definition of $T$ respectively.
\end{proof}
Standard conformal prediction methods do not provide simultaneous coverage guarantees for each batch when the number of batches is unknown in advance. Using a stricter coverage $\Bar{\alpha} < \alpha$ combined with union bound does not solve this issue, as it requires knowing the number of batches beforehand. We will show how e-values address this issue in Section \ref{sec:ville}.

\subsection{Fixed-size conformal sets with data-dependent coverage}

Traditional conformal prediction methods operate by fixing a coverage level $\alpha$ in advance, ensuring that the true label falls within the conformal set with probability at least $1-\alpha$. While this approach offers strong theoretical guarantees, it does not provide control over the size of the conformal sets, which can vary widely depending on the underlying data distribution.

Instead of fixing $\alpha$ a priori, we allow the coverage level $\tilde{\alpha}$ to adapt based on the data, tailored to yield conformal sets of desired size. We describe this adaptive approach in Section \ref{sec:post-hoc}. It enables more flexible conformal sets that are better aligned with the specific instance being predicted, potentially improving the interpretability and usability of conformal prediction in real-world applications.

\subsubsection{Motivating example: Medical diagnosis}

In healthcare, doctors often rely on machine learning models to help identify potential diseases based on patient symptoms, medical history, and test results. Conformal prediction methods enhance these models by providing sets of possible diagnoses with confidence guarantees, assisting doctors in prioritizing further tests or treatments. However, due to time constraints and resource limitations, doctors can typically consider only a fixed number of diagnoses (e.g., the top three or four).

While statistically sound, the traditional conformal prediction approach does not control set size, which can vary significantly. In some cases, the set may be impractically large, while in others, it may be too small, risking missed diagnoses.

\subsection{Conformal prediction under ambiguous ground truth}

We will also discuss a third application of conformal e-prediction, which naturally addresses a problem introduced by \cite{stutz2023conformal}.

In machine learning, labels are not always precisely known, particularly in scenarios where expert annotations are used. Instead of standard feature-label pairs $(X_i,Y_i)$, we may encounter $(X_i,Y_i^j)$ for $j=1,...,m$ where $m$ is the number of experts providing predictions. This framework arises, for instance, in medicine, where a patient $X_i$ is assessed by multiple experts, each predicting a diagnosis $Y_i^j$. Consequently, the samples become non-exchangeable, making traditional conformal prediction methods relying on rank-based statistics inapplicable.

This setting has been discussed by \cite{stutz2023conformal} under a framework they refer to as \emph{Monte Carlo conformal prediction}. In their approach, the calibration set consists of features $X_i \sim \mathbb{P}_X$, and for each $X_i$, $m$ labels $Y_i^j \sim \mathbb{P}_{Y | X_i}$, for $j=1,...,m$. Consequently, the calibration set ${\{(X_i,Y_i^j)\}}{\substack{j=1,...,m\\i=1,...,n}}$ no longer consists of exchangeable data, making the standard conformal prediction technique inapplicable.

To address this, \cite{stutz2023conformal} build on classical statistical results on p-variable averaging to provide a coverage guarantee of $1-2\alpha$. In Section \ref{sec:ambiguous-ground-truth}, we suggest an alternative based on e-values for which we can establish theoretical guarantees for achieving a coverage of $1-\alpha$.  Thus an e-value-based appproach can provide uncertainty quantification in this non-exchangeable setting.

\subsection{A brief overview of conformal e-prediction}

As we will show, compared to p-variables, on which most existing conformal prediction methods are based, e-variables offer greater flexibility and enable the construction of conformal sets in more complex settings.\footnote{We refer to ``p-variables'' and ``e-variables'' when discussing the underlying random variable, and ``p-values'' and ``e-values'' for the observed values.}  We begin with a brief introduction to e-values and their usage in conformal prediction.

\begin{definition}[e-variable]
    An \emph{e-variable} E is a nonnegative random variable that satisfies 
    \[
    \mathbb{E}[E] \le 1.
    \]
\end{definition}
Coupled with probabilistic inequalities, well-designed e-variables based on the calibration set and the test data point enable the construction of relevant conformal sets. For instance, using Markov's inequality, we obtain:
\[
\mathbb{P}(E < 1/\alpha) = 1-\mathbb{P}(E \ge 1/\alpha) \ge 1- \alpha\mathbb{E}[E] \ge 1- \alpha,
\]
which allows for the construction of conformal sets satisfying (\ref{cp_goal}).

\begin{remark}[Terminology]
    For clarity, we will use the term \emph{conformal p-prediction} for p-variable-based conformal prediction methods, rather than simply \emph{conformal prediction} as is common in the literature. \emph{Conformal e-prediction} refers to conformal prediction methods based on e-variables. The term \emph{conformal prediction} broadly denotes the construction of conformal sets for a test point using a calibration set, regardless of the method.  
\end{remark}
Initially, in the 1990s, conformal prediction methods were introduced using e-variables (under different names, as the term e-variable only recently emerged in the literature). Over the years, conformal prediction methods based on p-variables gradually gained popularity.  For a discussion on the advantages of e-variables and p-variables in conformal prediction, we recommend reading \cite{vovk2024conformaleprediction}, who provides a historical context on why p-variables came to dominate over e-variables in conformal prediction. While conformal p-prediction has been widely used, e-variables bring several fundamental advantages that can be advantageous. 
In this work, we leverage e-variables to produce valid conformal sets for the three case studies introduced earlier: batch anytime-valid conformal prediction, fixed-size conformal sets with data-dependent coverage, and conformal prediction under ambiguous ground truth.
Although any e-variable can be used, for concreteness we adopt the following e-variable: 
\begin{equation} 
\label{def_e_var}
    E = \frac{S_{n+1}}{\frac{1}{n+1}\sum_{i=1}^{n+1}S_i},
\end{equation}
which was introduced by \cite{balinsky2024EnhancingCP}, and which can be readily seen to have expectation equal to one under exchangeability.\footnote{For a formal proof, see \citet{balinsky2024EnhancingCP}.} Using this e-variable in conformal prediction we directly compare the test data point's score with the average of all the scores, rather than comparing their ranks.

\section{Batch Anytime-valid Conformal Prediction}
\label{sec:ville}

The methods derived from e-statistics are powerful tools for statistical inference due to their connection with fundamental probabilistic inequalities. As discussed previously, a key starting point is Markov's inequality, which provides an upper bound on the tail probabilities of nonnegative random variables. A stronger result generalizing Markov's inequality is an inequality due to \cite{ville1939}, which applies to nonnegative supermartingales and is central in sequential analysis and anytime-valid inference. To state Ville’s inequality rigorously, we first define the concepts of filtrations and (super)martingales; see \cite{Williams1991martingale} for a standard reference.

Formally, we fix a probability space $(\Omega,\mathcal{F},\mathbb{P})$, where $\Omega$ is a sample space, $\mathcal{F}$ is a $\sigma$-algebra, and $\mathbb{P}$ is a probability measure. We focus on real-valued random variables, defined as measurable functions mapping $\Omega$ to $\mathbb{R}$.

\begin{definition}[Filtration]
A \emph{filtration} is a sequence $\{\mathcal{F}_t\}_{t\ge0}$ of sub-$\sigma$-algebras of $\mathcal{F}$ such that $\mathcal{F}_s \subseteq  \mathcal{F}_t$ for all $s \le t$.
\end{definition}

\begin{definition}[Martingale]
Let $\{M_t\}_{t\ge0}$  be a sequence of random variables. The sequence $\{M_t\}_{t\ge0}$ is a \emph{martingale} for the filtration $\{\mathcal{F}_t\}_{t\ge0}$ if:
\begin{itemize}
    \item $M_t$ is $\mathcal{F}_t$-measurable for all $t \ge 0$,
    \item $\mathbb{E}[\lvert M_t \rvert] < +\infty $ for all $t \ge 0$,
    \item $\mathbb{E}[ M_{t+1} | \mathcal{F}_{t}] = M_t$ (almost surely) for all $t \ge 0$.
\end{itemize}

If the equality $\mathbb{E}[ M_{t+1} \mid \mathcal{F}_{t}] = M_t$ is replaced with the inequality $\mathbb{E}[ M_{t+1} \mid \mathcal{F}_{t}] \le M_t$, $\{M_t\}_{t\ge0}$ is said to be a \emph{supermartingale}. In particular, all martingales are supermartingales.
\end{definition}
If the filtration is clear from context, we may say that $\{M_t\}_{t\ge0}$ is a (super)martingale without explicitly specifying the filtration $\{\mathcal{F}_t\}_{t\ge0}$. In this work, we focus on supermartingales with an initial value of $M_0=1$; these are usually refered to as \emph{test supermartingales} (see, for example, \cite{waudbysmith2023estimating}). In this specific case, Ville's inequality can be stated as follows:

\begin{theorem}[Ville's Inequality]
Let $\{M_t\}_{t\ge0}$ be a nonnegative test supermartingale. Then, for any $\alpha \in (0,1)$:
\[
\mathbb{P}(\forall t \ge 0, \ M_t < 1/\alpha) \ge 1-\alpha.
\]
\end{theorem}
A straightforward proof follows from the upcrossing lemma and the supermartingale convergence theorem of \cite{doob1953}.

\begin{remark}[Ville’s Inequality with Stopping Times]
A useful alternative formulation of Ville’s inequality involves stopping times. Recall that a random variable $\tau$ taking values in $\mathbb{N} \cup \{\infty \}$ is a \emph{stopping time} with respect to a filtration $\{\mathcal{F}_t\}_{t\ge0}$ if $\{\tau \le t\} \in \mathcal{F}_t$ for all $t \ge 0$, meaning that the decision to stop at time $\tau$ depends only on the information available up to $t$. Given a nonnegative test supermartingale $\{M_t\}_{t\ge0}$, Ville’s inequality extends to stopping times as follows:
\[
\mathbb{P}(M_\tau < 1/\alpha) \ge 1-\alpha,
\]
for any stopping time $\tau$. This result ensures that the coverage guarantee holds regardless of when we stop the process based on the observed data. A proof of this equivalent reformulation can be found in Lemma 3 of \cite{howard2021timeuniform}.
\end{remark}
Ville’s inequality provides a strong guarantee for nonnegative test supermartingales, bounding the probability that the process exceeds a threshold $1/\alpha$ at any time step. This makes it a crucial tool for constructing batch anytime-valid conformal sets.

\subsection{Theoretical framework}

To apply Ville's inequality, we first need a carefully chosen nonnegative test supermartingale, which in turn requires defining an appropriate filtration. Throughout this section, we will consider the filtration of sigma-algebras generated by all random variables from the data batches obtained so far. More precisely:
\begin{align*}
    &\mathcal{F}_1 = \sigma(S_1^{1},...,S_{n_1}^1,S_{n_1+1}^1),\\
    &\mathcal{F}_2 = \sigma(S_1^{1},...,S_{n_1}^1,S_{n_1+1}^1,S_1^{2},...,S_{n_2}^2,S_{n_2+1}^2),
\end{align*}
and more generally: 
\[
\mathcal{F}_t = \sigma(S_1^{1},...,S_{n_1}^1,S_{n_1+1}^1,...,S_1^{t},...,S_{n_t}^t,S_{n_t+1}^t).
\]
In particular, $\mathcal{F}_0$ is the trivial sigma-algebra $\{ \emptyset, \Omega\}$. We can now define a sequence of random variables $\{M_t\}_{t \ge 0}$, which will be fundamental for defining our batch anytime-valid conformal sets. The definition of $M_t$ is inspired by the e-variable in Equation (\ref{def_e_var}) introduced by \cite{balinsky2024EnhancingCP}.

\begin{theorem}
\label{thm:martingale}
    For all $t \ge 0$, the sequence of random variables $\{M_t \}_{t \ge 0}$ defined by:
    \[
    M_t = \prod_{s=1}^t E_s
    \]
    where 
    \[
    E_s = \frac{S_{n_s+1}^s}{\frac{1}{n_s+1}\sum_{j=1}^{n_s+1} S_j^s}
    \]
    for all $s \ge 1$, is a nonnegative test supermartingale.
\end{theorem}
\begin{proof}
    First, by the definition of an empty product, we have that $M_0 = 1$. Now, since we have assumed in this work that the scores are positive, it is clear that $M_t \ge 0$ for all $t \ge 0$. The only thing left to prove is that $\{M_t \}_{t \ge 0}$ is a supermartingale. Let $t \ge 1$. We have:
    \[
        \mathbb{E}[M_t | \mathcal{F}_{t-1}] = \mathbb{E}\left[\prod_{s=1}^t E_s \Big| \mathcal{F}_{t-1}\right] = \underbrace{\left(\prod_{s=1}^{t-1} E_s \right)}_{= M_{t-1}} \mathbb{E}[E_t | \mathcal{F}_{t-1}]
    \]
    and see that $\mathbb{E}[E_t | \mathcal{F}_{t-1}]= 1 $ since we assumed that, within a given batch, the data are exchangeable, conditional on the previous data batches. A proof of this fact can be found in \cite{balinsky2024EnhancingCP}. Therefore, $\{M_t \}_{t\ge0}$ is a martingale, and thus, in particular, a supermartingale.
\end{proof}
Therefore, by applying Ville's inequality to the test supermartingale $\{M_t\}_{t \ge 0}$, we deduce the following corollary.

\begin{corollary}
\label{cor:bav-cs}
    For $t \ge 1$, let $\hat{C}_t$ be the subset of $\mathbb{R}$ defined by:
    \begin{equation}
    \label{bav-cs}
    \hat{C}_t := \left\{ v \in \mathbb{R}_+ : \prod_{s=1}^{t-1} E_s \times \frac{v}{\frac{1}{n_t+1}\left(\sum_{j=1}^{n_t} S_j^t + v\right)} < 1/\alpha \right\}.
    \end{equation}
    Thus $\{\hat{C}_t \}_{t \ge 0}$ is a sequence of batch anytime-valid conformal sets.
\end{corollary}
In the definition of $\hat{C}_t$, $v$ serves as a placeholder for the random variable $S_{n_t+1}^t$. 
The set $\hat{C}_t$ is constructed to ensure that $S_{n_t+1}^t$ falls within it with high probability.

\begin{remark}
\label{rmk:universal-portfolio}
    We constructed the martingale $M_t$ in Theorem \ref{thm:martingale} by simply taking the product of the e-variables $E_s$. While the product martingale has certain desirable properties, alternative constructions are possible. In particular, Theorem \ref{thm:martingale} extends naturally, and one can see that the cumulative product $M_t = \prod_{s=1}^t (1-\lambda_s + \lambda_s E_s)$, where $\lambda_t \in [0,1]$ is any measurable function of $E_1,...,E_{t-1}$, is also a martingale. This construction enjoys appealing theoretical properties, especially when using $\lambda_t=\underset{\lambda \in [0,\gamma]}{\arg\max} \frac{1}{t-1} \sum_{s=1}^{t-1} \log(1-\lambda+\lambda E_s)$ for some $\gamma \in (0,1]$. In this paper, we focus on the ``all-in" process with $\lambda_t=1$ for simplicity, but exploring these alternative processes would be worthwhile. For further details, we refer to \cite{waudbysmith2023estimating,wang2024ebacktesting,ramdas2024hypothesistestingevalues}. This construction is closely related to the universal portfolio algorithm proposed by \cite{cover1991up}; see also \cite{vovk1990aggregating,cover1996up,vovk1998up,orabona2023tightconcentrations,ryu2024onconfidencesequences}.
\end{remark}

\subsection{Experiments}

In real-world machine learning applications, models are often deployed in dynamic environments where data arrives in a sequential or distributed manner. Unlike traditional static datasets where calibration is performed globally, real-world deployments often require incremental calibration as new users interact with the system.

For our experiments, we use the Federated Extended MNIST (FEMNIST) dataset, a standard benchmark in federated learning and image classification \citep{caldas2018leaf}. FEMNIST extends the EMNIST dataset \citep{cohen2017emnist}, itself derived from MNIST \citep{lecun98mnist}, and includes 62 classes: digits (0-9) and uppercase and lowercase letters (A-Z, a-z). Originally designed for federated learning, where data are distributed non-i.i.d.\ across clients, we adapt it by treating each writer as a data batch. The dataset consists of 3,597 clients—each representing a unique writer—with varying numbers of samples, introducing natural heterogeneity and imbalance. We leverage this structure to simulate scenarios where data arrives in groups with inherent variability. The dataset is split into a training set of approximately 650,000 images and a test set of about 165,000 images, with distinct writers in each set to prevent overlap. Each image is a $28\times28$ grayscale representation, consistent with the original MNIST format.

\begin{figure}[h!]
    \centering
    \includegraphics[width=0.6\textwidth]{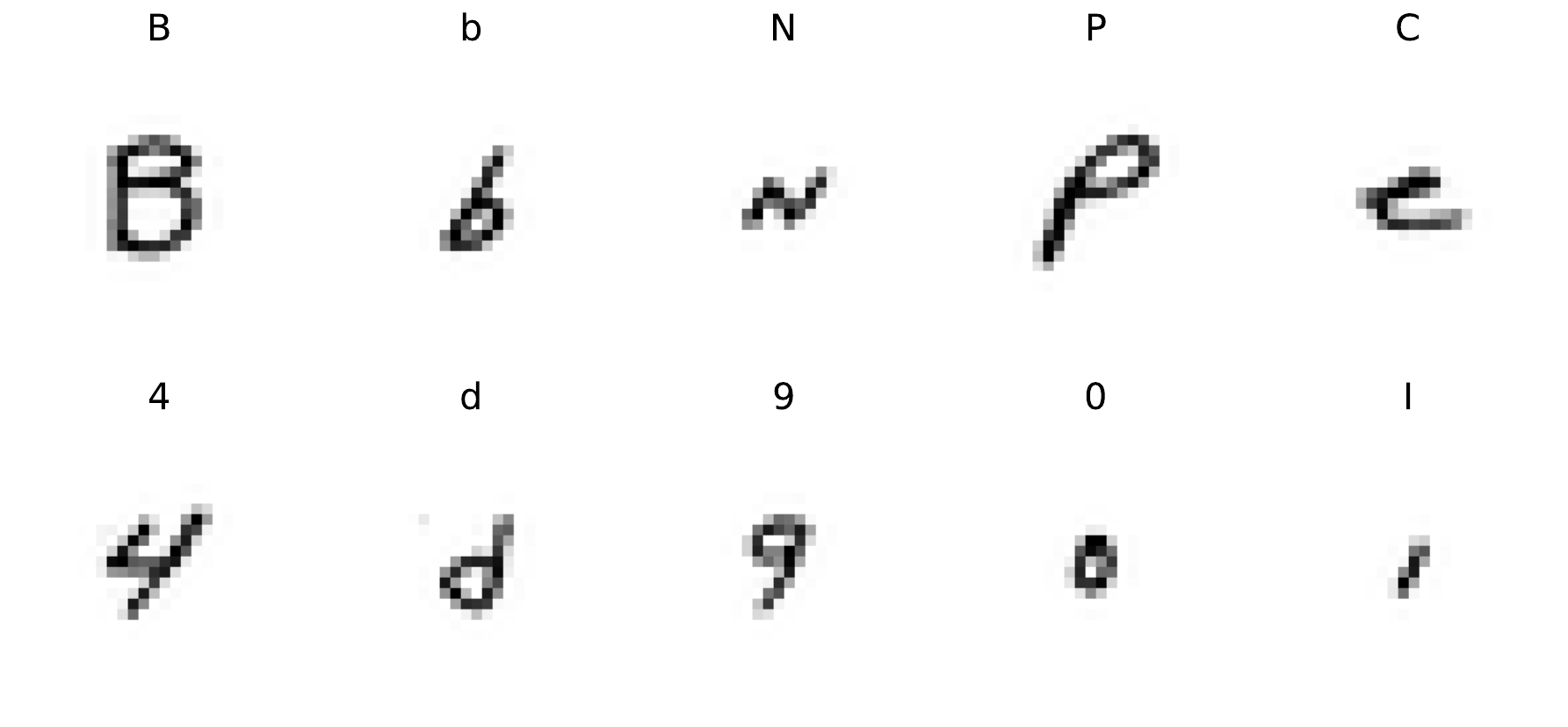}
    \vspace{-1.2em}
    \caption{Some images from FEMNIST with their associated class label.}
    \label{fig:femnist_image}
\end{figure}

We train a convolutional neural network $f$, inspired by LeNet \citep{lecun1998lenet} and adapted for our 62-class recognition task. This simple model serves as a black-box predictor to evaluate the effectiveness of our conformal prediction method\footnote{Since conformal prediction is model-agnostic, its validity holds regardless of the underlying model. Therefore, we use a standard, well-understood model rather than a state-of-the-art architecture. This choice ensures that the observed coverage behavior stems from the conformal method itself rather than from improvements in the base model’s accuracy.}. Details on the architecture and training parameters are provided in Appendix \ref{app:ville}. On the test set, $f$ achieves 87.6\% accuracy.

The test set consists of multiple writers, each contributing a variable number of image samples. To implement our batch anytime-valid conformal prediction method, we proceed as follows. We randomly select $T=50$ writers from the test set to serve as data batches. Since the dataset is finite, we run our method for a fixed number of batches, 50 in this case, but the approach extends to any $T$. We then split the test set into two equal parts: a calibration set (50\% of the original test data) and a final test set (the remaining 50\%).

Conformal sets are constructed sequentially. At step $t=1$, for the first selected writer, we extract the calibration scores $S_1^1,...,S_{n_1}^1$, where $n_1$ is the number of samples from that writer in the calibration set. Using these scores, we construct a conformal prediction set following (\ref{bav-cs}). Next, we sample a test point $S_{n_1+1}^1$ from the final test set corresponding to the same writer and evaluate the e-variable $E_1$. We repeat this process for each subsequent writer ($t=2,3,...,50$).

The choice of the score function $S$ is key to obtaining informative conformal sets. In particular, $S$ should be chosen so that the martingale $M_t$ in Theorem \ref{thm:martingale} does not decay to zero too quickly. When $M_t$ is close to zero, the conformal set $\hat{C}_{t+1}$ expands significantly, as seen in (\ref{bav-cs}). To mitigate this, we define the score function as
\[
S(X,Y)=\frac{1}{p_f(Y|X)^{1/4}},
\]
where $X$ is an image, $Y\in{1,...,62}$ is a possible label, and $p_f(Y|X)$ is the probability assigned to $Y$ by the pre-trained model $f$. This choice ensures that lower predicted probabilities correspond to higher scores, aligning with the intuition that less certain predictions should be assigned greater importance. The exponent $1/4$ further sharpens the scores by amplifying the impact of particularly low-confidence predictions. We set the exponent to $1/4$ as it empirically yields favorable martingale dynamics, preventing $M_t$ from collapsing too quickly to zero in practice. This choice is based on its consistent performance on our dataset, though a deeper theoretical understanding remains an open question.

We set the coverage level to $\alpha=0.15$. The experiment is repeated 100 times, resampling both the calibration and final test sets at each iteration. Averaging over all runs, we obtain an empirical coverage of 0.94, exceeding the target level of $1-\alpha=0.85$.

\begin{figure}[h!]
    \centering
    \includegraphics[width=0.78\textwidth]{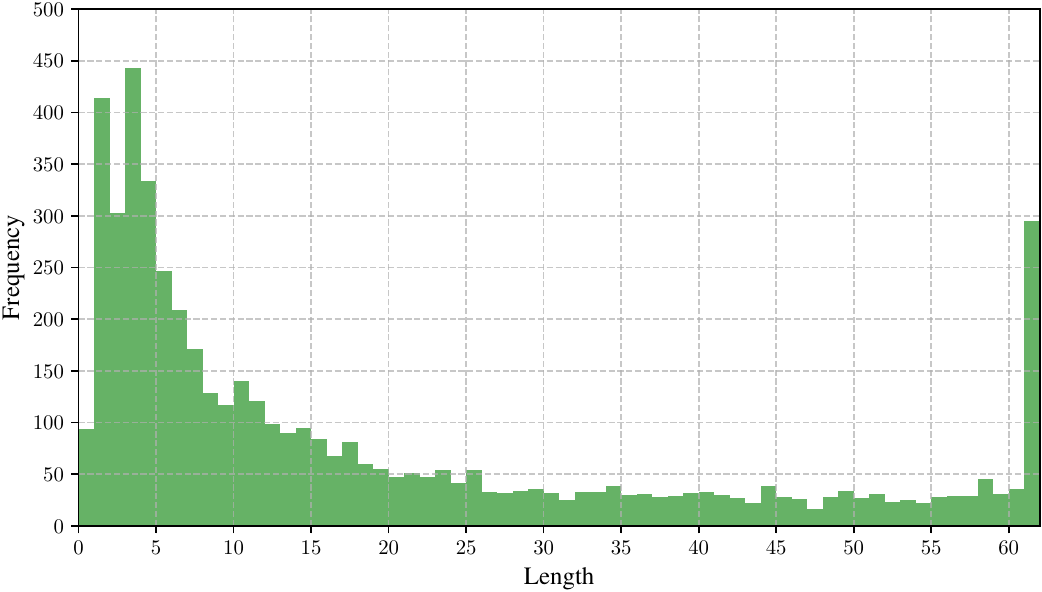}
    \caption{Histogram of conformal set sizes obtained across all $T=50$ data batches and 100 iterations. The distribution illustrates the variability in set sizes, highlighting the proportion of informative and trivial sets.}
    \label{fig:hist-anytime-valid}
\end{figure}

This over-coverage can be attributed to several factors. First, Ville’s inequality is not necessarily tight, as is the case for Markov’s inequality. Second, the coverage guarantee is anytime-valid, meaning it must hold over an infinite sequence of batches. In contrast, our experiments are limited to $T=50$ batches due to the finite amount of data available, which naturally leads to a higher observed coverage in practice.

This over-coverage is not necessarily a drawback, provided the conformal sets remain reasonably small and informative. Figure \ref{fig:hist-anytime-valid} shows the distribution of conformal set sizes across $T=50$ data batches and 100 iterations. Most sets are relatively compact and convey meaningful information. However, a limitation of this method is that some sets can be trivial---in this case, those of size 62. Despite this, most remain informative and non-trivial, underscoring the approach’s practicality. It is worth noting that the method is model-agnostic and ensures coverage but does not control conformal set size; stronger models typically yield smaller sets. Appendix \ref{app:ville} provides a more detailed breakdown of conformal set sizes per batch, along with examples of batch anytime-valid conformal set~sequences.

\section{Fixed-Size Conformal Sets with Data-Dependent Coverage}
\label{sec:post-hoc}

Compared to p-variables, on which most existing conformal prediction methods are based, e-variables offer a more generalized form of coverage: they enable data-dependent coverage. This is why it is somewhat unfair to directly compare conformal p-prediction and conformal e-prediction, using a fixed coverage level~$\alpha$. E-variables provide a stronger guarantee, known as post-hoc validity, which allows for more flexible and adaptive inference. Further details on post-hoc validity and its connection to e-variables can be found in the following works: \cite{wang2022fdr, xu2024postselectioninference,grunwald2024beyond,ramdas2024hypothesistestingevalues,koning2024posthocalphahypothesistesting}. Our discussion will specifically draw from the work by \cite{koning2024posthocalphahypothesistesting}.

\subsection{Theoretical framework}

\begin{definition}[Post-hoc p-variable, Definition 2 from \cite{koning2024posthocalphahypothesistesting}]
\label{def:post_hoc_p-var}
    We say that a nonnegative random variable $P$ is a \emph{post-hoc p-variable} if
    \[
    \underset{\tilde{\alpha}}{\sup} \ \mathbb{E}\left[ \frac{\mathbb{P}(P \le \tilde{\alpha} \mid \tilde{\alpha})}{\tilde{\alpha}} \right] \le 1,
    \]
    where the supremum is over every random variable $\tilde{\alpha} > 0$ that may not be independent from $P$.
\end{definition}
P-variables are nonnegative random variables that satisfy the condition that their size distortion, $\frac{\mathbb{P}(P \le \alpha)}{\alpha}$, for a given coverage level $\alpha$, remains below 1. When the coverage is data-dependent, the size distortion itself becomes a random variable. The interpretation of Definition \ref{def:post_hoc_p-var} of \emph{post-hoc} p-variables is that certain distortions are permitted, provided they remain controlled in expectation. Remarkably, post-hoc p-variables are exactly the inverses of e-variables:

\begin{theorem}[Theorem 2 from \cite{koning2024posthocalphahypothesistesting}]
\label{thm:post-hoc-iff-e-var}
    $P$ is a post-hoc p-variable if and only if $\mathbb{E}[1/P] \le 1$.
\end{theorem}
Therefore, conformal e-prediction methods yield data-dependent coverage guarantees, in the following sense:

\begin{proposition}
    Consider a calibration set $\{(X_i,Y_i)\}_{i=1,\dots,n}$ and a test data point $(X_{n+1},Y_{n+1})$ such that $(X_1,Y_1),\dotsc,(X_n,Y_n),(X_{n+1},Y_{n+1})$ are exchangeable. Let $\tilde{\alpha}$ be any coverage level that may depend on this data. Then we have that:
    \begin{equation}
    \label{post-hoc}
    \mathbb{E}\left[ \frac{\mathbb{P}(Y_{n+1} \not\in \hat{C}_n^{\tilde{\alpha}}(X_{n+1}) \mid \tilde{\alpha})}{\tilde{\alpha}} \right] \le 1,    
    \end{equation}
    where 
    \begin{equation}
    \label{eq:post-hoc-cs}
    \hat{C}_n^{\tilde{\alpha}}(x) := \left\{y :     \frac{S(x,y)}{\frac{1}{n+1}\left(\sum_{i=1}^{n}S(X_i,Y_i) + S(x,y) \right)} < 1/\tilde{\alpha} \right\}.
    \end{equation}
\end{proposition}

\begin{proof}
    Recall that the random variable 
    \[
    E = \frac{S(X_{n+1},Y_{n+1})}{\frac{1}{n+1}\sum_{i=1}^{n+1}S(X_i,Y_i)}
    \]
    defined in Equation (\ref{def_e_var}) is an e-variable \citep[see][]{balinsky2024EnhancingCP}. Therefore, by Theorem \ref{thm:post-hoc-iff-e-var}, $P = 1/E$ is a post-hoc p-variable. The only thing left to see is that
    \begin{align*}
        P \le \tilde{\alpha} &\Longleftrightarrow  E \ge 1/\tilde{\alpha}\\
        &\Longleftrightarrow \frac{S(X_{n+1},Y_{n+1})}{\frac{1}{n+1}\sum_{i=1}^{n+1}S(X_i,Y_i)}\ge 1/\tilde{\alpha}\\
        &\Longleftrightarrow Y_{n+1} \not\in \hat{C}_n^{\tilde{\alpha}}(X_{n+1}) 
    \end{align*}
    by definition of $\hat{C}_n^{\tilde{\alpha}}$.
\end{proof}
In practice, practitioners can define a random variable $\tilde{\alpha}$ that depends only on the observed data $\{(X_i,Y_i)\}_{i=1,\dots,n}$ and $X_{n+1}$, but not on $Y_{n+1}$. When $\tilde{\alpha} = \alpha$ is a constant that does not depend on data, the guarantee (\ref{post-hoc}) can be rewritten as 
\[
\mathbb{P}(Y_{n+1} \not\in \hat{C}_n^{\alpha}(X_{n+1})) \le \alpha,
\]
which is exactly (\ref{cp_goal}). More generally, when $\tilde{\alpha}$ is well-concentrated around its mean, we can use a first-order Taylor expansion to obtain:
\[
\mathbb{E}\left[ \frac{\mathbb{P}(Y_{n+1} \not\in \hat{C}_n^{\tilde{\alpha}}(X_{n+1}) \mid \tilde{\alpha})}{\tilde{\alpha}} \right] \approx \frac{ \mathbb{E}[ \mathbb{P}(Y_{n+1} \not\in \hat{C}_n^{\tilde{\alpha}}(X_{n+1}) \mid \tilde{\alpha}) ]}{\mathbb{E}[\tilde{\alpha}]} = \frac{ \mathbb{P}(Y_{n+1} \not\in \hat{C}_n^{\tilde{\alpha}}(X_{n+1}))}{\mathbb{E}[\tilde{\alpha}]}
\]
and derive meaningful guarantees of the form:
\begin{equation}
\label{data-dependent1}
    \mathbb{P}(Y_{n+1} \in \hat{C}_n^{\tilde{\alpha}}(X_{n+1})) \ge 1 - \mathbb{E}[\tilde{\alpha}].
\end{equation}
Since $\mathbb{E}[\tilde{\alpha}]$ is not known in practice, we can use concentration inequalities to obtain useful guarantees. For example, assume that $\tilde{\alpha}$ is $\sigma$-sub-Gaussian. Then, $\tilde{\alpha}$ satisfies the deviation inequality $\mathbb{P}(\tilde{\alpha} \le \mathbb{E}[\tilde{\alpha}] -t) \le \exp\left(-\frac{t^2}{2\sigma^2}\right)$ for all $t \in \mathbb{R}$, and we obtain guarantees of the~form:
\begin{equation}
\label{data-dependent2}
    \mathbb{P}(Y_{n+1} \in \hat{C}_n^{\tilde{\alpha}}(X_{n+1})) \ge 1 - \tilde{\alpha} - \sigma\sqrt{2\log(1/\delta)},
\end{equation}
with probability at least $1-\delta$.

Therefore, e-variables can be utilized to derive conformal prediction guarantees that adapt to the data. This enables, for instance, the construction of conformal sets with a specified size. In the classification setting, suppose we aim to obtain conformal sets of size (at most) $C$. Traditional conformal prediction methods do not allow for such control. However, by leveraging e-variables, we can define
\begin{equation}
\label{eq:def-alpha-tilde}
\tilde{\alpha} := \inf \left\{ \alpha \in (0,1) : \# \left\{y : \frac{S(X_{n+1},y)}{\frac{1}{n+1}\left(\sum_{i=1}^{n}S(X_i,Y_i)+S(X_{n+1},y)\right)} < 1/\alpha \right\} \le C \right\}
\end{equation}
and establish guarantees of the form (\ref{post-hoc}), (\ref{data-dependent1}), and (\ref{data-dependent2}), while ensuring that the conformal sets are of size (at most) $C$. This approach enables practitioners to make informed decisions based on the observed $\tilde{\alpha}$. It is only feasible with e-variables, as they provide post-hoc guarantees in contrast with the fixed-level $\alpha$ guarantees of conformal p-prediction.

\subsection{Experiments}

We use the same dataset and the same pre-trained model $f$ as in Section \ref{sec:ville}. For our experiments, we randomly sample a calibration set of size 2000 from the test set and select a test data point $(X_{n+1},Y_{n+1})$ from the remaining data points. Our goal is to obtain conformal sets of size (at most) $C$, with a data-dependent coverage. To achieve this, we determine $\tilde{\alpha}$ based on Equation (\ref{eq:def-alpha-tilde}), that depends on the calibration set and the test feature $X_{n+1}$. For ease of implementation, we define a discrete set of candidate values: $\mathbb{A}=\{0.01,0.02,0.03,\dots,0.29,0.3\}$, and select
\begin{equation}
\label{eq:def-alpha-tilde-A}
\tilde{\alpha} := \inf \left\{ \alpha \in \mathbb{A} : \# \left\{y : \frac{S(X_{n+1},y)}{\frac{1}{n+1}\left(\sum_{i=1}^{n}S(X_i,Y_i)+S(X_{n+1},y)\right)} < 1/\alpha \right\} \le C \right\}.
\end{equation}

\begin{figure}[h]
    \centering
    \subfloat[\( C = 3 \)]{
        \includegraphics[width=0.48\linewidth]{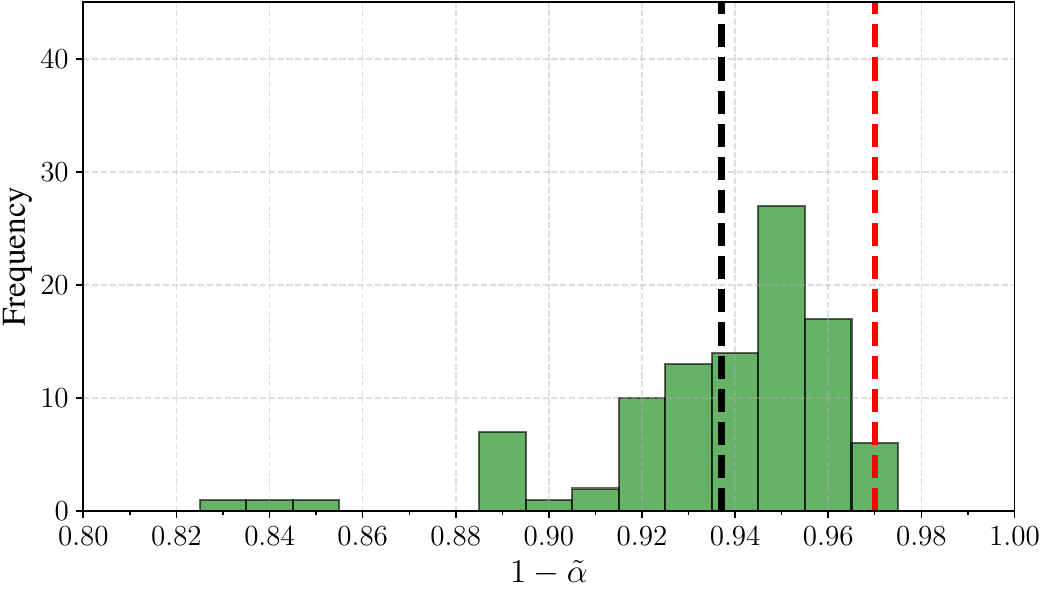}
    }
    \hfill
    \subfloat[\( C = 5 \)]{
        \includegraphics[width=0.48\linewidth]{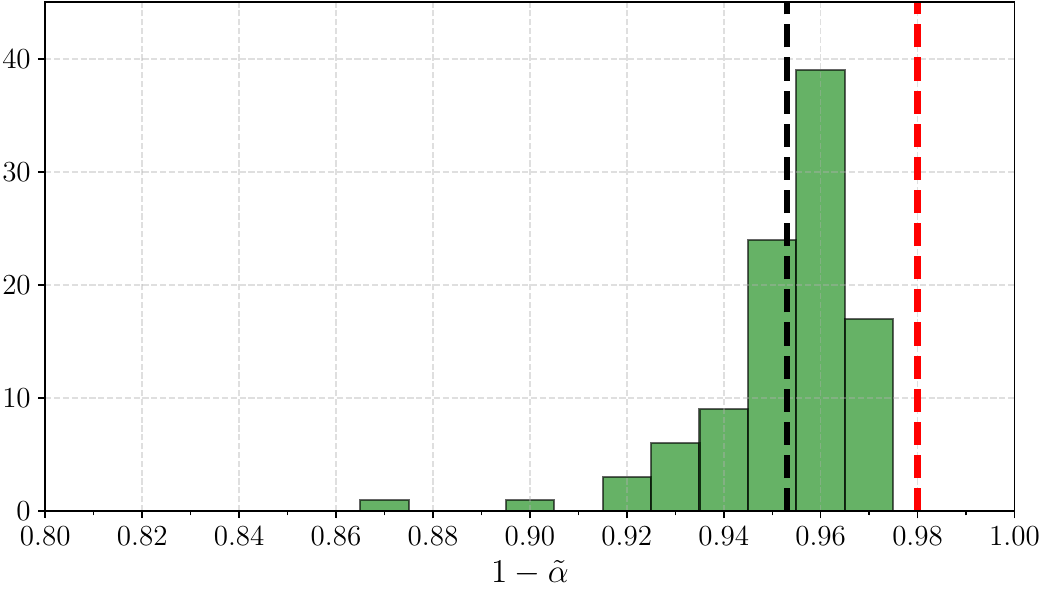}
    }
    \caption{Histogram of $1-\tilde{\alpha}$ for \( C = 3 \) and \( C = 5 \), computed across 100 iterations. The black dashed line represents the expected coverage level, $1-\mathbb{E}[\tilde{\alpha}]$, while the red dashed line corresponds to the empirical coverage probability, $\mathbb{P}(Y_{n+1} \in \hat{C}_n^{\tilde{\alpha}}(X_{n+1}))$, both estimated over the 100 iterations.}
    \label{fig:post-hoc-alpha}
\end{figure}

The score function we use is the cross-entropy, formally defined in Equation (\ref{cross-entropy}). We average the results over 100 iterations, each with a randomly sampled calibration set. In each iteration, we compute a data-dependent coverage level $\tilde{\alpha}$ and construct conformal sets of size (at most) $C$ using (\ref{eq:post-hoc-cs}). The values of $\tilde{\alpha}$ for $C=3$ and $C=5$ are shown in Figure \ref{fig:post-hoc-alpha}. Empirically, we observe that inequality (\ref{data-dependent1}) holds, supporting the first-order Taylor approximation used in our method to derive guarantees (\ref{data-dependent1}) and (\ref{data-dependent2}). We provide examples of the obtained conformal sets and illustrate how their size evolves as $\alpha$ varies in Appendix~\ref{app:post-hoc}.

\begin{figure}[h!]
    \centering
    \includegraphics[width=\textwidth]{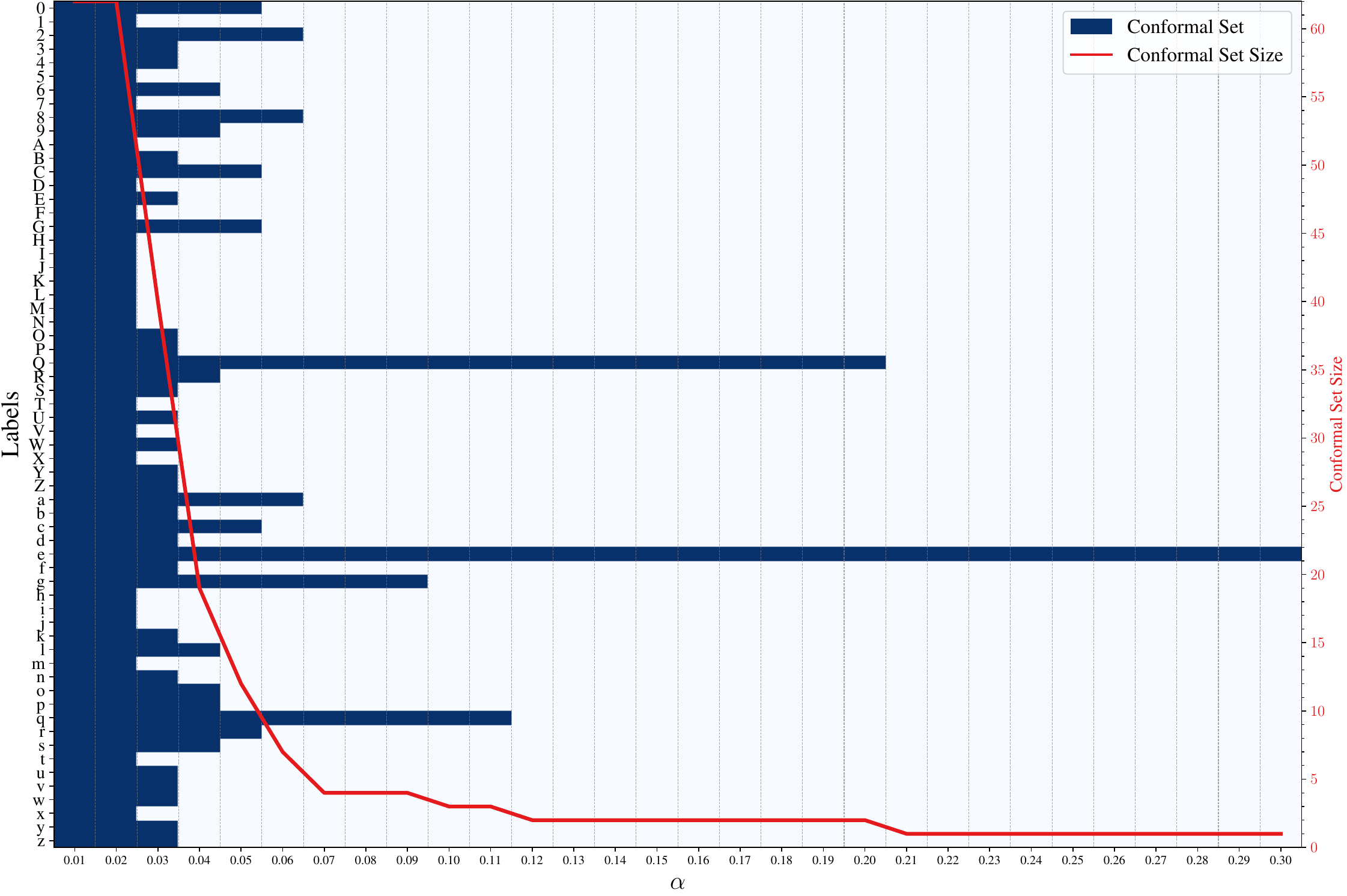}
    \caption{Example of conformal sets obtained with varying $\tilde{\alpha}$.}
    \label{fig:example3}
\end{figure}

This method introduces a conceptual inversion of conformal prediction: instead of fixing a coverage level and potentially obtaining conformal sets of impractical size, we set a maximum conformal set size threshold $C$ up front. We then obtain guarantees for an observed coverage level $\tilde{\alpha}$, which is data-dependent and adaptive in the sense that it varies with the test feature $X_{n+1}$. This allows practitioners to make informed decisions based on the observed coverage while ensuring that the resulting conformal sets remain meaningful.

We present a sample result in Figure \ref{fig:example3}. Specifically, this sample corresponds to a random draw of the calibration set and test data point. We construct the associated conformal set for each $\alpha \in \mathbb{A}$ using (\ref{eq:def-alpha-tilde-A}). The x-axis represents the different values of $\alpha$, while the y-axis corresponds to the possible labels. For each $\alpha$, the conformal sets are shown as dark blue cells. Additionally, we plot in red the size of the conformal sets as a function of $\alpha$.

The method allows us to select $\tilde{\alpha}$ based on the calibration data and test image $X_{n+1}$. For instance, if we aim to obtain conformal sets with a size (at most) $C=5$, applying Equation (\ref{eq:def-alpha-tilde}) yields $\tilde{\alpha}=0.07$ in the example of Figure \ref{fig:example3}. We provide additional sample results in Appendix \ref{app:post-hoc}.

\section{Conformal Prediction under Ambiguous Ground Truth}
\label{sec:ambiguous-ground-truth}

The conformal e-prediction framework also provides valid guarantees in Monte Carlo conformal prediction, a problem introduced by \cite{stutz2023conformal}. When the ground truth is ambiguous and $m \ge 1$ labels are sampled from $\mathbb{P}_{Y|X_i}$ for each feature $X_i$, it is possible to account for all $m$ labels while obtaining conformal prediction guarantees. Instead of selecting a single label and discarding the rest to enforce exchangeable data, an approach that leverages all $m$ labels reduces variability in the conformal prediction guarantees. \cite{stutz2023conformal} propose a method to achieve this based on a classical result by \cite{ruschendorf1982pvalues} and later independently established by \cite{meng1994pvalues} on averaging arbitrary p-variables, which yields a p-variable up to a factor of 2. Consequently, this approach ensures a coverage level of $1-2\alpha$, with this factor of 2 being theoretically unavoidable. Constructing conformal sets using e-variables, however, guarantees a theoretical coverage of at least $1-\alpha$.

\subsection{Monte Carlo conformal p-prediction \citep{stutz2023conformal}}

We will write $S_i^j := S(X_i,Y_i^j)$ for $i=1,...,n$ and $j=1,...,m$. We will also write $S_{n+1} := S(X_{n+1},Y_{n+1})$. We reformulate the main result by \cite{stutz2023conformal}, adapted for negatively-oriented scores instead of positively-oriented ones, which is the primary focus of their paper.

\begin{theorem}
\label{thm:mc-cp_p-var}
    The following conformal set:
    \begin{align}
    \label{cs_mc-cp-p}
    \hat{C}_n(x) &:= \left\{y : \frac{1}{mn} \sum_{j=1}^m\sum_{i=1}^n \mathbb{1}_{\{S_i^j \le S(x,y)\}} \le \frac{\lceil m(1-\alpha)(n+1)\rceil-1}{mn} \right\}\\
    &= \left\{y : S(x,y) \le \text{\emph{quantile}}\left(\frac{\lceil m(1-\alpha)(n+1)\rceil-1}{mn} ; \frac{1}{mn}\sum_{j=1}^m\sum_{i=1}^n \delta_{S_i^j} \right) \right\} 
    \end{align}
    with $\delta_z$ denoting a point mass at $z$, satisfies property (\ref{cp_goal}) with coverage $1-2\alpha$ instead of~$1-\alpha$:
    \[
    \mathbb{P}(Y_{n+1} \in \hat{C}_n(X_{n+1})) \geq 1 - 2\alpha.
    \]
\end{theorem}

\begin{proof}
Let $\alpha \in (0,1)$. We have:
\begin{align*}
    \frac{1}{mn}\sum_{j=1}^m\sum_{i=1}^{n}\mathbb{1}_{\{S_i^j \le S_{n+1}\}} \le \frac{\lceil m(1-\alpha)(n+1)\rceil-1}{mn}
    &\Leftrightarrow\sum_{j=1}^m\sum_{i=1}^{n}\mathbb{1}_{\{S_i^j \le S_{n+1}\}} \le \lceil m(1-\alpha)(n+1)\rceil-1 \\
    &\Leftrightarrow \sum_{j=1}^m\sum_{i=1}^{n}\mathbb{1}_{\{S_i^j \le S_{n+1}\}} < m(1-\alpha)(n+1)\\
    &\Leftrightarrow mn-\sum_{j=1}^m\sum_{i=1}^{n}\mathbb{1}_{\{S_i^j > S_{n+1}\}} < m(1-\alpha)(n+1)\\
    &\Leftrightarrow \sum_{j=1}^m\sum_{i=1}^{n}\mathbb{1}_{\{S_i^j > S_{n+1}\}} > mn - m(1-\alpha)(n+1)\\
    &\Leftrightarrow \frac{1}{m}\sum_{j=1}^m \frac{\sum_{i=1}^n \mathbb{1}_{\{S_i^j>S_{n+1}\}}+1}{n+1} > \alpha,
\end{align*}
and we will show that the last inequality holds with probability $\ge 1-2\alpha$. First, see that for a fixed $j=1,...,m$, the random variables $S_1^j,...,S_n^j,S_{n+1}$ are exchangeable. Therefore, the random variables 
\[
P_j := \frac{\sum_{i=1}^n \mathbb{1}_{\{S_i^j>S_{n+1}\}}+1}{n+1}
\]
are p-variables (see Lemma \ref{lma:cp_p-variables} in Appendix \ref{app:intro-conformal-prediction}). Let $\Bar{P}:= \frac{1}{m} \sum_{j=1}^m P_j$ be the arithmetic mean of these p-variables. Then, by \cite{ruschendorf1982pvalues, meng1994pvalues}, $2\Bar{P}$ is a p-variable: 
\[
    \forall \alpha \in (0,1), \ \mathbb{P}(2\Bar{P}\le \alpha) \le \alpha \Longleftrightarrow \forall \alpha \in (0,1), \  \mathbb{P}(\Bar{P} > \alpha) \ge 1-2\alpha,
\]
which concludes the proof.
\end{proof}
Even if there exist other ways of combining p-variables (see, for example, \cite{vovk2020combining}), the advantage of averaging p-variables using the arithmetic mean is that it preserves the interpretation of conformal sets as quantiles of the empirical cumulative distribution function. However, this approach comes at the cost of achieving coverage of $1-2\alpha$ instead of $1-\alpha$. We now show that moving beyond the traditional framework of conformal p-prediction based on rank-based statistics and instead leveraging e-variables yields conformal sets with coverage $1-\alpha$.

\subsection{Monte Carlo conformal e-prediction}

E-variables are particularly convenient to handle. Specifically, it is evident that if $E_1,...,E_m$ are e-variables, then their arithmetic mean:
    \begin{equation}
    \label{eq:avg-evar}
    \Bar{E} := \frac{1}{m} \sum_{j=1}^m E_j  
    \end{equation}
is also an e-variable.
We keep the same notation as before. For each $j \in \{1,...,m \}$, the random variables $S_1^j,...,S_n^j,S_{n+1}$ are exchangeable so the random variables
\[
E_j := \frac{S_{n+1}}{\frac{1}{n+1}(\sum_{i=1}^{n}S_i^j+S_{n+1})}
\]
are e-variables. By applying Markov's inequality to $\Bar{E}$ defined in (\ref{eq:avg-evar}), we derive a new way to construct conformal sets in Monte Carlo conformal prediction:

\begin{theorem}
    The following conformal set:
    \begin{align}
    \label{cs_mc-cp-e}
        \hat{C}_n(x) &:= \left\{y : \frac{1}{m} \sum_{j=1}^m \frac{S(x,y)}{\frac{1}{n+1}(\sum_{i=1}^n S(X_i,Y_i^j) + S(x,y))} < 1/\alpha \right\}
    \end{align}
    satisfies property (\ref{cp_goal}).
\end{theorem}
One possible interpretation is as follows: assume one of the $m$ experts, denoted $j_0$, predicts that $E_{j_0}\ge1/\alpha$, meaning that the expert believes that the test point's score is abnormally high. This expert might be incorrect. The conformal set (\ref{cs_mc-cp-e}) smoothes the predictions by averaging all the predictions from expert $j_0$ and the remaining $m-1$ experts.

\subsection{Experiments}

In this section, we base our experiments on the CIFAR-10 dataset \citep{krizhevsky2009learning}. CIFAR-10 is a widely used benchmark dataset in computer vision, consisting of 60,000 color images of size $32\times32$ pixels, categorized into 10 classes: airplane, automobile, bird, cat, deer, dog, frog, horse, ship, and truck. The dataset is divided into 50,000 training images and 10,000 test images, with an equal number of samples per class.

We train an EfficientNetB0 model \citep{tan2019efficientnet} using cross-entropy loss, optimized with stochastic gradient descent. The learning rate is initialized at 0.1 and adjusted with a cosine annealing schedule. Training is conducted over 100 epochs with a batch size of 512, incorporating data augmentation techniques. The resulting classifier $f$ achieves an accuracy of 98.6\% on the training set and 91.1\% on the test set. Once trained, we treat $f$ as a black-box model for generating predictions. 

The Monte Carlo conformal prediction method is particularly relevant in cases where the ground truth is ambiguous. For this reason, we use the CIFAR-10H dataset \citep{peterson2019cifar10h}, which extends CIFAR-10 by incorporating human-labeled annotations for the 10,000 images of the test set. For each image, human annotators provide class labels, offering an additional source of ground truth that can be used to assess model performance in comparison to human judgment. This human-annotated version of CIFAR-10 is particularly useful for tasks where model uncertainty is of interest.

\begin{figure}[h!]
    \centering
    \includegraphics[width=\textwidth]{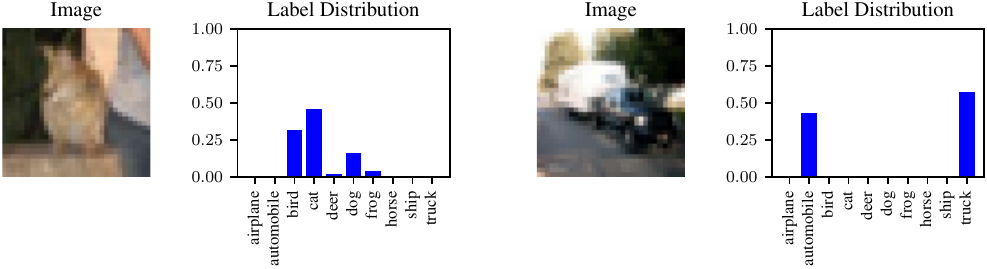}
    \caption{Some images from CIFAR-10H with ambiguous ground truth, along with their label distributions.}
    \label{fig:cifar10h_image}
\end{figure}

We filter the test dataset to include only ambiguous examples. Specifically, we retain only images with label distributions where at least two labels have a probability $\ge 0.1$. This results in a filtered test dataset of size 857. This dataset is then divided into a calibration set comprising 30\% of the data and a final test set with the remaining 70\%.

On the filtered test dataset, the model $f$ achieves an accuracy of 71.41\%. Similar to \cite{stutz2023conformal}, which selects a coverage matching the accuracy of the pre-trained model, we set $\alpha=0.3$ to stay comparable with the base model. We use cross-entropy (\ref{cross-entropy}) as the score function $S$.

Figures \ref{fig:hist-e} and \ref{fig:hist-p} show the coverage and conformal set size obtained with conformal e-prediction and conformal p-prediction respectively, comparing $m=1$ and $m=20$ experts across 200 random splits between the calibration and final test sets. We overlap the histograms for $m=1$ and $m=20$ using a fixed bin width. Increasing the number of experts $m$ reduces the variability in coverage and conformal set sizes observed across calibration and test splits, for both methods. This is consistent with the observation of \cite{stutz2023conformal} in the case of conformal p-prediction. The conformal sets obtained with e-variables exhibit larger coverage than the target $1-\alpha$ and are reasonably larger than the conformal sets obtained with conformal p-prediction. We also note that although Monte Carlo conformal p-prediction has a theoretical coverage of $1-2\alpha$, we empirically observe a coverage of $1-\alpha$, consistent with the empirical findings of \cite{stutz2023conformal}.

\begin{figure}[h!]
    \centering
    \includegraphics[width=\textwidth]{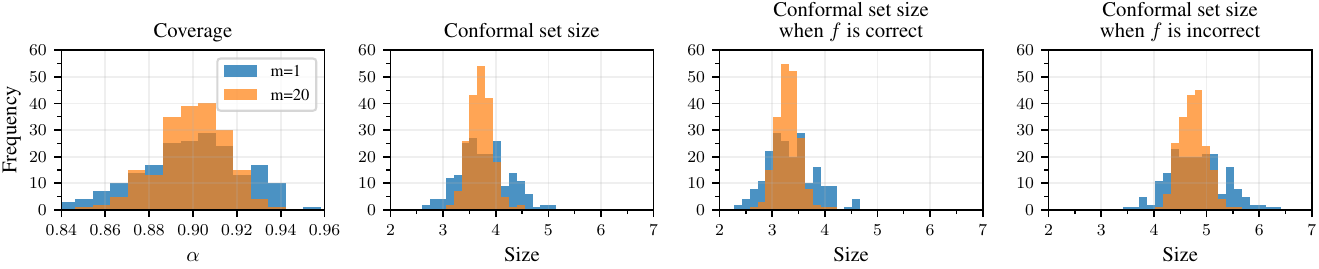}
    \caption{Comparison of coverage and conformal set sizes when using e-variables in Monte Carlo conformal prediction with $m=1$ or $m=20$ experts, with $\alpha=0.3$, from Theorem \ref{cs_mc-cp-e}.}
    \label{fig:hist-e}
\end{figure}

\begin{figure}[h!]
    \centering
    \includegraphics[width=\textwidth]{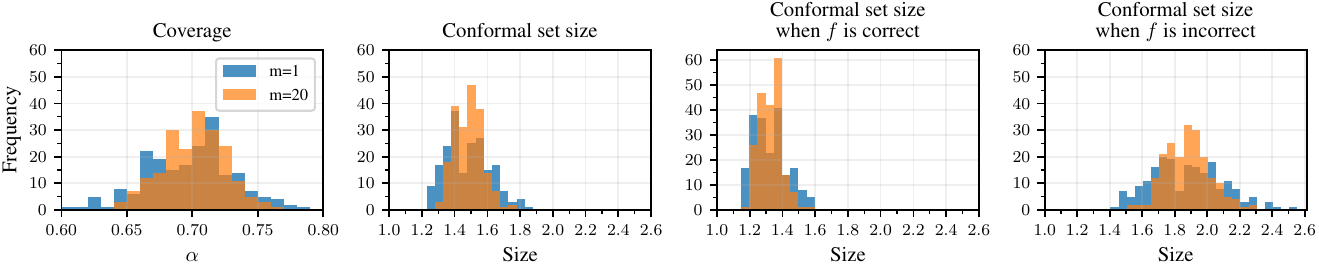}
    \caption{Comparison of coverage and conformal set sizes when using p-variables in Monte Carlo conformal prediction with $m=1$ or $m=20$ experts, with $\alpha=0.3$, from Theorem \ref{cs_mc-cp-p}.}
    \label{fig:hist-p}
\end{figure}

Note that to achieve a theoretical coverage guarantee of 0.70 in Monte Carlo conformal p-prediction, one can set $\alpha=0.15$, as this method theoretically ensures a coverage of $1-2\alpha$ rather than $1-\alpha$. Empirically, this would yield conformal sets with a coverage of 0.85, almost matching those obtained in Monte Carlo conformal e-prediction, which achieve an empirical coverage of approximately 0.90.

Monte Carlo conformal e-prediction is a theoretical alternative to Monte Carlo conformal p-prediction, providing a theoretical coverage guarantee of $1-\alpha$. However, in practice, Monte Carlo conformal p-prediction also achieves an empirical coverage of $1-\alpha$ rather than the theoretically guaranteed $1-2\alpha$, and seems to achieve higher-quality predictions. The reason behind this discrepancy remains unclear and is still an open problem.

\section{Conclusion}

In this work, we highlighted the flexibility of conformal e-prediction for uncertainty quantification. While traditional conformal p-prediction will continue to provide an important tool, in particular in standard statistical settings where data are exchangeable and a fixed coverage level $\alpha$ is employed, conformal e-prediction provides an appealing alternative for more complex scenarios, particularly when data are not fully exchangeable. We exhibited this flexibility in three settings where e-variables yield theoretically well-grounded conformal prediction guarantees: batch anytime-valid conformal prediction, fixed-size conformal sets with data-dependent coverage, and conformal prediction under ambiguous ground truth. These results highlight the potential of conformal e-prediction as a flexible tool for uncertainty quantification for complex problems in machine learning.

We believe there is much more to explore in conformal e-prediction. One ongoing challenge is the choice of the score function, which plays a crucial role in determining the size of conformal sets. Selecting an optimal score function is particularly important in settings like batch anytime-valid prediction, where it directly impacts the dynamics of the martingale.

Beyond this, further investigation into the applications introduced in this paper could yield valuable insights. In batch anytime-valid conformal prediction, an interesting question is whether we can use alternatives to Ville’s inequality that provide tighter bounds. While Ville’s inequality holds for supermartingales, we could explore inequalities that apply to martingales as well, given that $M_t$ introduced in Theorem \ref{thm:martingale} is not only a supermartingale but also a martingale. Another avenue is the choice of the martingale itself—could different constructions lead to stronger results? In particular, it would be interesting to investigate whether incorporating $\lambda_t$ terms that depend on the data up to time $t-1$, as explained in Remark \ref{rmk:universal-portfolio}, could lead to improved performance. 

\section*{Acknowledgements}
The authors would like to thank Eugène Berta, Nabil Boukir, Sacha Braun, Aymeric Capitaine, Mahmoud Hegazy, Nick Koning, Antoine Scheid, and Ian Waudby-Smith for their constructive reflections on conformal prediction and e-values, which have provided valuable insights.

Funded by the European Union (ERC-2022-SYG-OCEAN-101071601).
Views and opinions expressed are however those of the author(s) only and do not
necessarily reflect those of the European Union or the European Research Council
Executive Agency. Neither the European Union nor the granting authority can be
held responsible for them. 

This publication is part of the Chair ``Markets and Learning,'' supported by Air Liquide, BNP PARIBAS ASSET MANAGEMENT Europe, EDF, Orange and SNCF, sponsors of the Inria Foundation.

This work has also received support from the French government, managed by the National Research Agency, under the France 2030 program with the reference ``PR[AI]RIE-PSAI" (ANR-23-IACL-0008).

\nocite{vovk2021evalues}

\vskip 0.2in
\bibliographystyle{plainnat}
\bibliography{references}

\newpage

\nocite{jin2023selection}

\appendix

\section{Background on Conformal p-Prediction}
\label{app:intro-conformal-prediction}

In this appendix, we review key concepts in conformal (p-)prediction, introduced by \cite{vovk2005algorithmiclearning}, to provide a comprehensive foundation. In our work, we focus on split conformal prediction \citep{papadopoulos2002inductivecm}, and we explore the applications of conformal prediction methods in machine learning, a broad and active area of research (see, for example, \cite{balasubramanian2014conformal,romano2019conformalized, sadinle2019leastambiguous,romano2020classification, angelopoulos2021uncertainty, stutz2022learning,huang2023uncertaintyquantificationovergraph}). For a recent overview of conformal prediction, we refer the reader to \cite{angelopoulos2023gentle} and \cite{angelopoulos2024theoreticalfoundationsconformalprediction}.

\begin{definition}[Exchangeability]
Random variables $S_1,...,S_n$ are said to be \emph{exchangeable} if: for any permutation $\sigma \in \mathfrak{S}_n$, the distribution of the random vector $(S_1,...,S_n)$ is the same as that of $(S_{\sigma(1)},...,S_{\sigma(n)})$.
\end{definition}

Exchangeability is a weaker property than the common notion of independent and identically distributed (i.i.d.) random variables. The fundamental idea of conformal p-prediction is based on a simple observation: since the scores $S_i = S(X_i,Y_i)$ for $i=1,...,n,n+1$, are exchangeable, their order\footnote{The notion of order is well-defined if we assume that the $S_i$ are almost surely distinct. Otherwise, we can use an appropriate random tie-breaking rule and the results still hold.} follows a uniform distribution. In particular, this implies that:
\[
    \mathbb{P}(S_{n+1} \le \text{the } \lceil(1-\alpha)(n+1)\rceil \text{ smallest of } S_1,...,S_{n+1}) \ge 1-\alpha,
\]
and it turns out that we can rewrite the inequality above as
\[
    \mathbb{P}(S_{n+1} \le \text{the } \lceil(1-\alpha)(n+1)\rceil \text{ smallest of } S_1,...,S_{n}) \ge 1-\alpha
\]
by artificially defining the $n+1$ smallest value among $S_1,...,S_n$ to be $+ \infty$. Therefore, the following conformal set:
\begin{equation}
\label{cs-p1}
\hat{C}_n(x) := \{y : S(x,y) \le \text{the } \lceil(1-\alpha)(n+1)\rceil \text{ smallest of } S(X_1,Y_1),...,S(X_n,Y_n) \}
\end{equation}
effectively satisfies (\ref{cp_goal}). Note that we can rewrite $\hat{C}_n$ using the empirical cumulative distribution function of $S(X_1,Y_1),...,S(X_n,Y_n)$ as follows:
\begin{equation}
\label{cs-p2}
\hat{C}_n(x) = \left\{y : \frac{1}{n}\sum_{i=1}^n \mathbb{1}_{\{S(X_i,Y_i) \le S(x,y)\}} \le \frac{\lceil(1-\alpha)(n+1)\rceil}{n} \right\}. 
\end{equation}
$\hat{C}_n$ is also commonly rewritten as:
\begin{equation}
\label{cs-p3}
\hat{C}_n(x) = \left\{y : S(x,y) \le \text{quantile}\left(\frac{\lceil(1-\alpha)(n+1)\rceil}{n} ; \frac{1}{n}\sum_{i=1}^n \delta_{S(X_i,Y_i)} \right) \right\}, 
\end{equation}
where the quantile function returns $+\infty$ when the input is $\ge 1$.
The interpretation of (\ref{cs-p1}), (\ref{cs-p2}), and (\ref{cs-p3}) is that, with high probability, the test value $S(X_{n+1},Y_{n+1})$ cannot significantly exceed the values $S(X_i,Y_i)$ in the calibration set. 

The formulation (\ref{cs-p2}) can be reinterpreted through the lens of p-variables. We recall the definition of p-variables:

\begin{definition}[p-variable]
    A \emph{p-variable} $P$ is a nonnegative random variable that satisfies 
    \[
    \mathbb{P}(P \le \alpha) \le \alpha
    \]
for all $\alpha \in (0,1)$.
\end{definition}

\begin{remark}
\label{rmk:size_distortion}
Equivalently, a nonnegative random variable $P$ is a p-variable if, for all $\alpha \in (0,1)$, its \emph{size distortion} $\mathbb{P}(P \le \alpha)/\alpha$ at level $\alpha$ is at most 1. This equivalent definition is especially meaningful in the discussion of Section \ref{sec:post-hoc}, where we leverage the post-hoc statistical properties that e-variables provide, in contrast to p-variables.
\end{remark}

\begin{lemma}[Reformulation of (\ref{cs-p2}) with p-variables]
\label{lma:cp_p-variables}
    The following random variable:
    \[
    \frac{\sum_{i=1}^n\mathbb{1}_{\{S_i > S_{n+1}\}}+1}{n+1}
    \]
    is a p-variable.
\end{lemma}
\begin{proof}
    Let $\alpha \in (0,1)$. We want to show that $\mathbb{P}\left(\frac{\sum_{i=1}^n\mathbb{1}_{\{S_i > S_{n+1}\}}+1}{n+1} \le \alpha \right) \le \alpha$, which is equivalent to $\mathbb{P}\left(\frac{\sum_{i=1}^n\mathbb{1}_{\{S_i > S_{n+1}\}}+1}{n+1} > \alpha \right) \ge 1-\alpha$. Note that:
    \begin{align*}
        \frac{\sum_{i=1}^n\mathbb{1}_{\{S_i > S_{n+1}\}}+1}{n+1} > \alpha
        &\Leftrightarrow \sum_{i=1}^{n}\mathbb{1}_{\{S_i > S_{n+1}\}} > \alpha(n+1)-1\\
        &\Leftrightarrow n-\sum_{i=1}^{n}\mathbb{1}_{\{S_i \le S_{n+1}\}} > \alpha(n+1)-1\\
        &\Leftrightarrow \sum_{i=1}^{n}\mathbb{1}_{\{S_i \le S_{n+1}\}} < (1-\alpha)(n+1)\\
        &\Leftrightarrow \sum_{i=1}^{n}\mathbb{1}_{\{S_i \le S_{n+1}\}} < \lceil(1-\alpha)(n+1)\rceil \text{\quad since $\sum_{i=1}^{n}\mathbb{1}_{\{S_i \le S_{n+1}\}} \in \mathbb{N}$}\\
        &\Leftrightarrow \frac{1}{n}\sum_{i=1}^{n}\mathbb{1}_{\{S_i \le S_{n+1}\}} < \frac{\lceil(1-\alpha)(n+1)\rceil}{n},
    \end{align*}
    and this last event indeed occurs with probability $\ge 1-\alpha$ since the conformal set (\ref{cs-p2}) satisfies Property (\ref{cp_goal}), which concludes the proof.
\end{proof}
Conformal p-prediction not only ensures that $\mathbb{P}(Y_{n+1} \in \hat{C}_n(X_{n+1})) \ge 1-\alpha$, but also guarantees that $\mathbb{P}(Y_{n+1} \in \hat{C}_n(X_{n+1})) \le 1-\alpha+\frac{1}{n+1}$, provided there are almost surely no ties between the scores.

\section{Details on Section \ref{sec:ville}}
\label{app:ville}

\subsection{Network architecture and training details}

The model $f$ consists of two convolutional layers followed by three fully connected layers. The first convolutional layer applies 6 filters of size $5\times5$ to the input grayscale images, followed by a ReLU activation and $2\times2$ max pooling. The second convolutional layer applies 16 filters of size $5\times5$, followed by another ReLU activation and $2\times2$ max pooling. The output of the second convolutional layer is flattened and passed through three fully connected layers: 120 neurons with ReLU activation, 84 neurons with ReLU activation, and finally 62 output neurons corresponding to the number of classes. The model is trained using cross-entropy loss and SGD with a learning rate of 0.1 for 100 epochs, using cosine annealing and a batch size of 512. The model achieves an accuracy of 88.9\% on the training set and 87.6\% on the test set.

\subsection{Additional plots}

Figure \ref{fig:average_lengths} shows the sizes of the conformal sets obtained for each data batch. Interestingly, the size distribution varies across batches, reflecting not only the influence of the method and the batch position in the sequence but also an intrinsic dependence on the data within each batch. We also empirically observe that the conformal sets tend to grow larger as the position $t$ in the sequence increases. This highlights a limitation of the method: the more we seek simultaneous guarantees, the weaker these guarantees become.

\begin{figure}[h!]
    \centering
    \includegraphics[width=\textwidth]{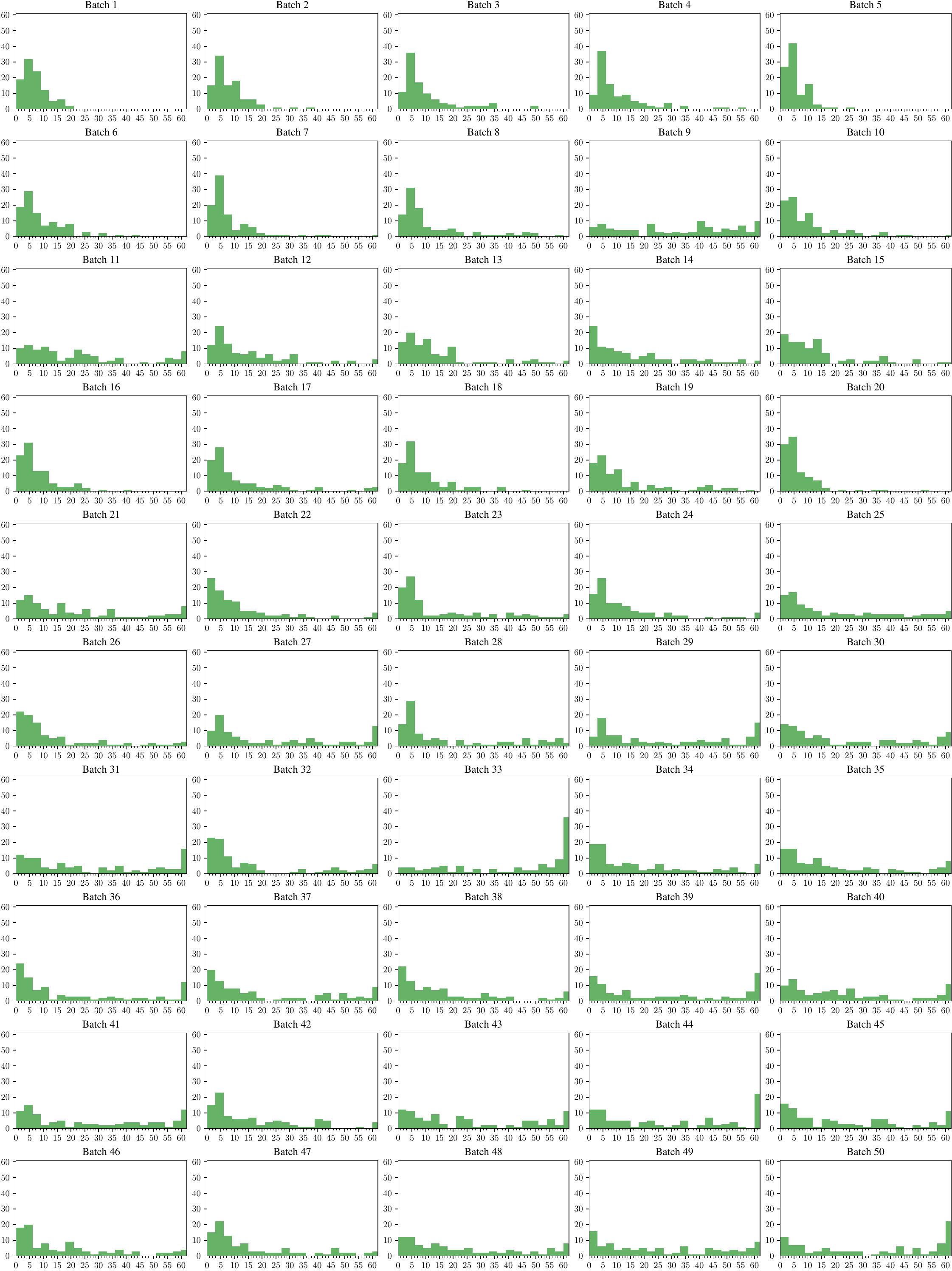}
    \caption{Histograms of conformal set sizes per batch, aggregated over 100 iterations.}
    \label{fig:average_lengths}
\end{figure}

We also visualize examples of batch anytime-valid conformal prediction sets in Figures \ref{fig:bav-conformal-sequence5}, \ref{fig:bav-conformal-sequence3}, \ref{fig:bav-conformal-sequence2}, \ref{fig:bav-conformal-sequence4}, and \ref{fig:bav-conformal-sequence}. The x-axis represents the time steps, while the y-axis corresponds to the different labels. Black cells indicate the conformal sets. For each time step, the true label is marked with a red star. For completeness, we also display the model $f$’s predictions using green triangles. Additionally, we plot the corresponding martingales $M_t$ in Figures \ref{fig:bav-martingale5}, \ref{fig:bav-martingale3}, \ref{fig:bav-martingale2}, \ref{fig:bav-martingale4}, and \ref{fig:bav-martingale}. As expected, the value of the martingale $M_t$ significantly influences the quality of the conformal set at time $t$. When the martingale is close to zero, the conformal sets become larger and less informative.

Interestingly, Figures \ref{fig:bav-conformal-sequence} and \ref{fig:bav-martingale} illustrate a case where the sequence of batch anytime-valid conformal sets does not fully cover the true labels. This occurs at batches $t=49$ and $t=50$.

\begin{figure}[h!]
    \centering
    \includegraphics[width=.9\textwidth]{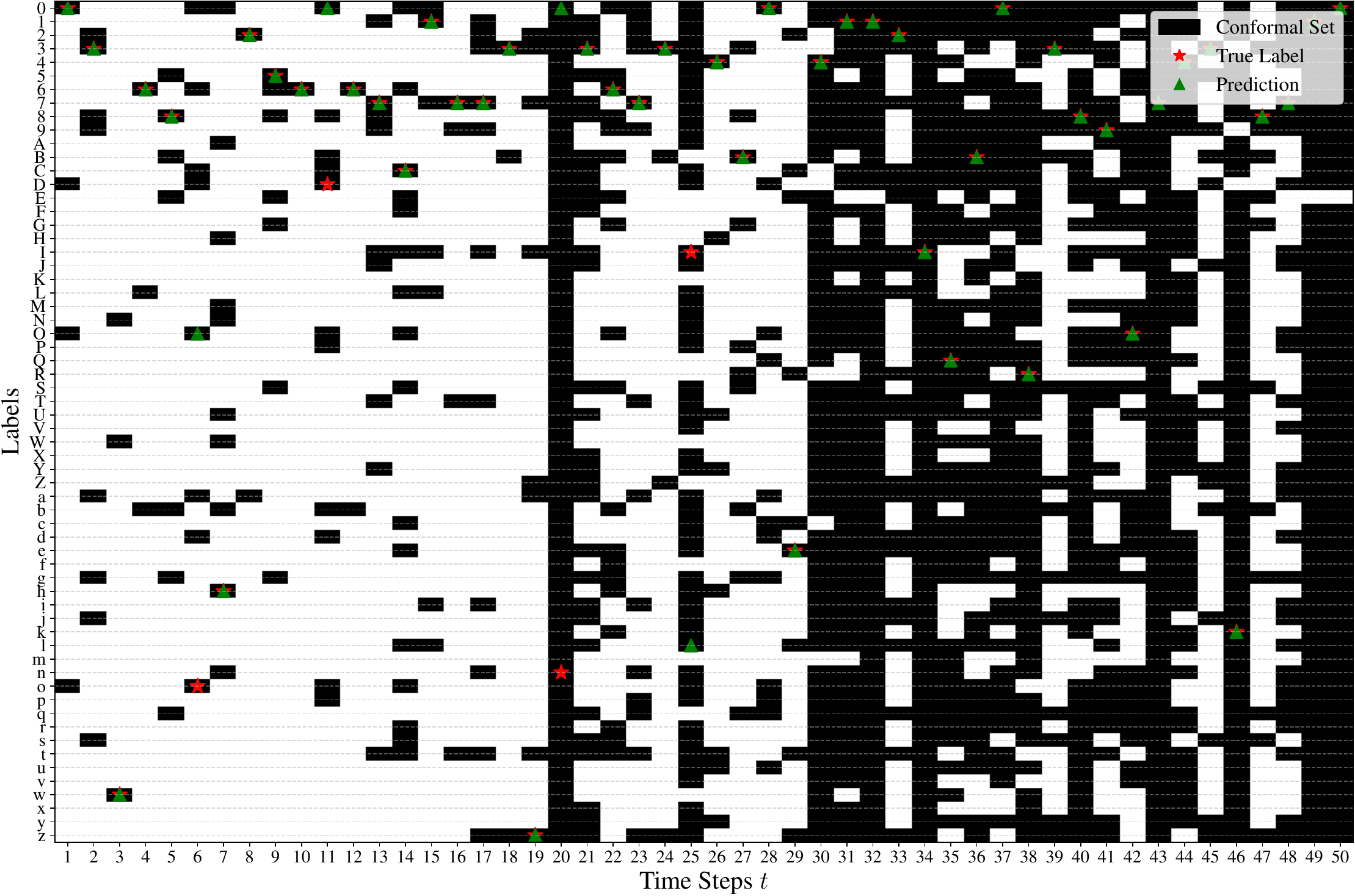}
    \caption{Example 1 of sequence of batch anytime-valid conformal sets.}
    \label{fig:bav-conformal-sequence5}
\end{figure}

\begin{figure}[h!]
    \centering
    \includegraphics[width=.9\textwidth]{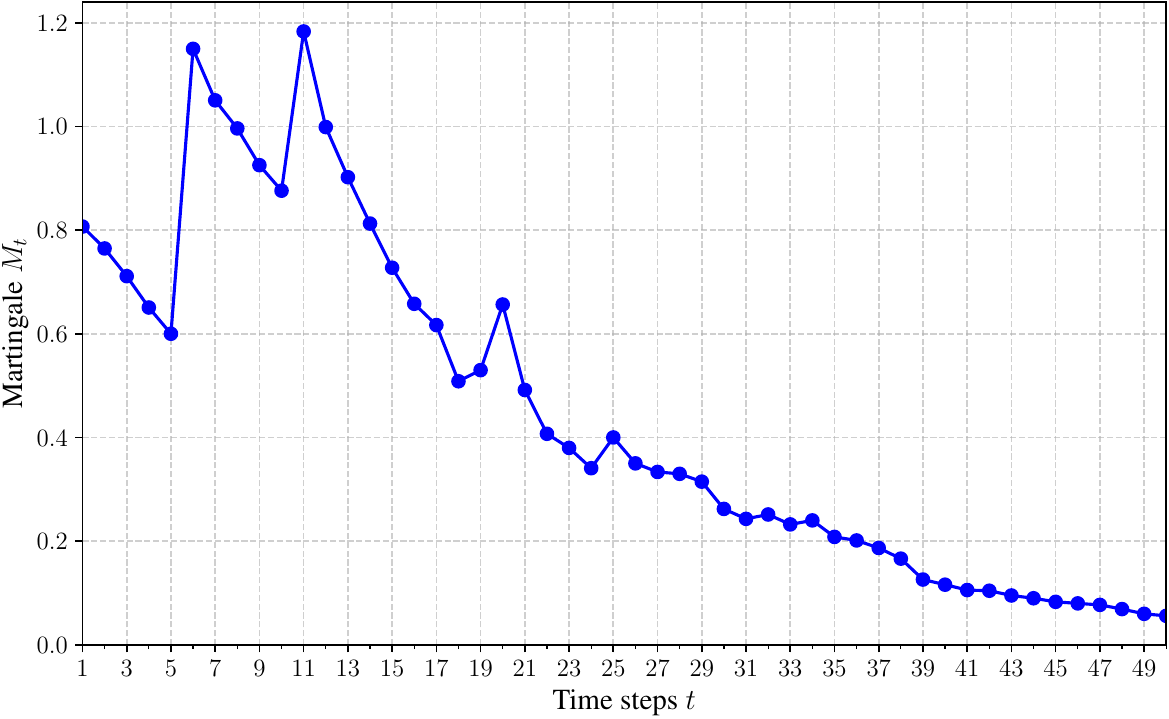}
    \caption{Sample path of the martingale $M_t$ associated with the sequence of batch anytime-valid conformal sets from Example 1.}
    \label{fig:bav-martingale5}
\end{figure}

\begin{figure}[h!]
    \centering
    \includegraphics[width=.9\textwidth]{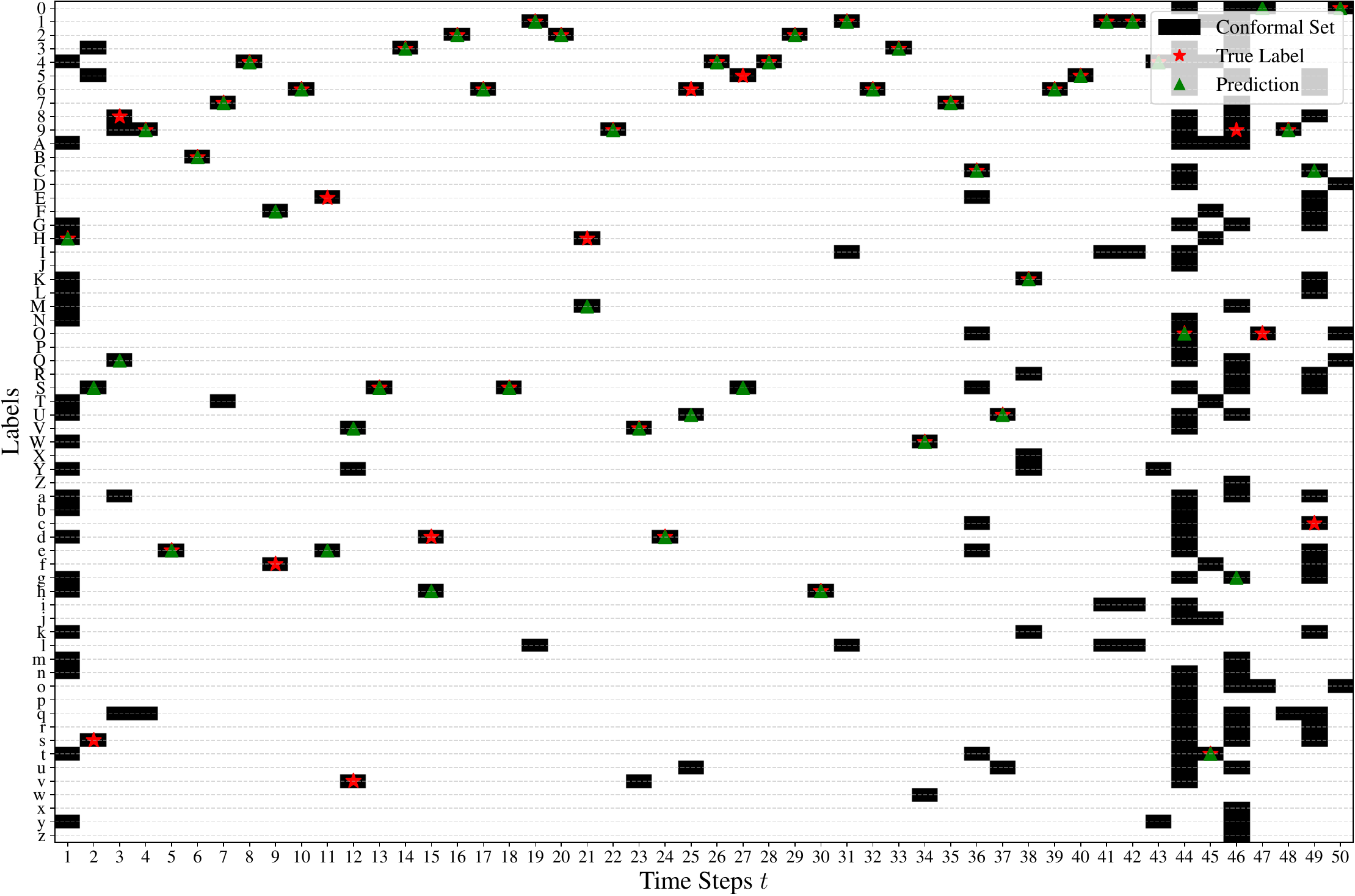}
    \caption{Example 2 of sequence of batch anytime-valid conformal sets.}
    \label{fig:bav-conformal-sequence3}
\end{figure}

\begin{figure}[h!]
    \centering
    \includegraphics[width=.9\textwidth]{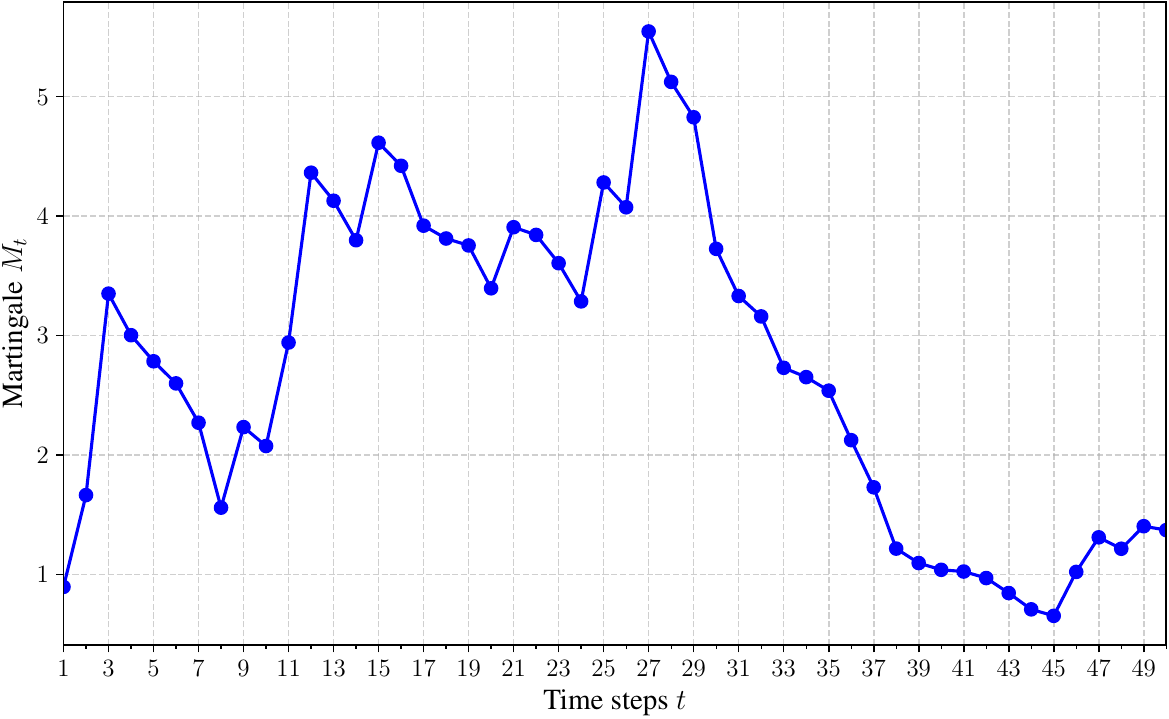}
    \caption{Sample path of the martingale $M_t$ associated with the sequence of batch anytime-valid conformal sets from Example 2.}
    \label{fig:bav-martingale3}
\end{figure}

\begin{figure}[h!]
    \centering
    \includegraphics[width=.9\textwidth]{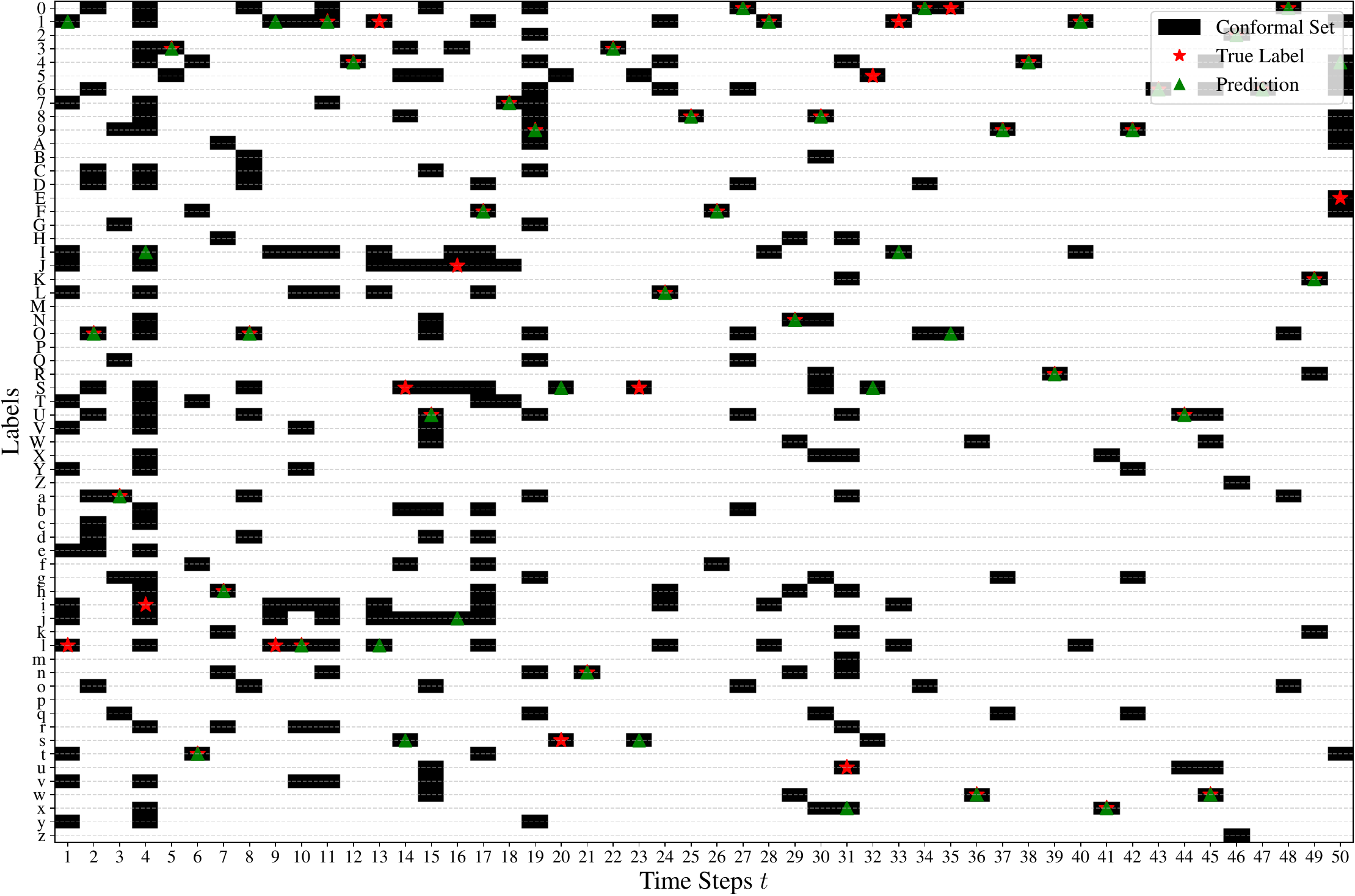}
    \caption{Example 3 of sequence of batch anytime-valid conformal sets.}
    \label{fig:bav-conformal-sequence2}
\end{figure}

\begin{figure}[h!]
    \centering
    \includegraphics[width=.9\textwidth]{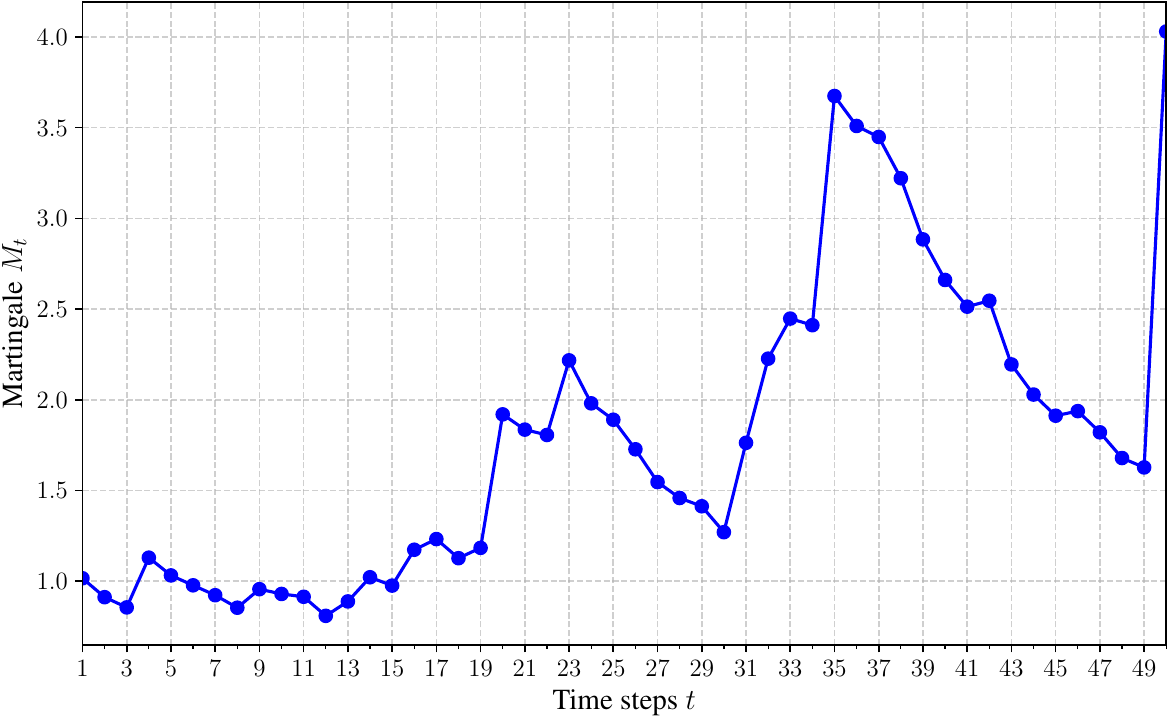}
    \caption{Sample path of the martingale $M_t$ associated with the sequence of batch anytime-valid conformal sets from Example 3.}
    \label{fig:bav-martingale2}
\end{figure}

\begin{figure}[h!]
    \centering
    \includegraphics[width=.9\textwidth]{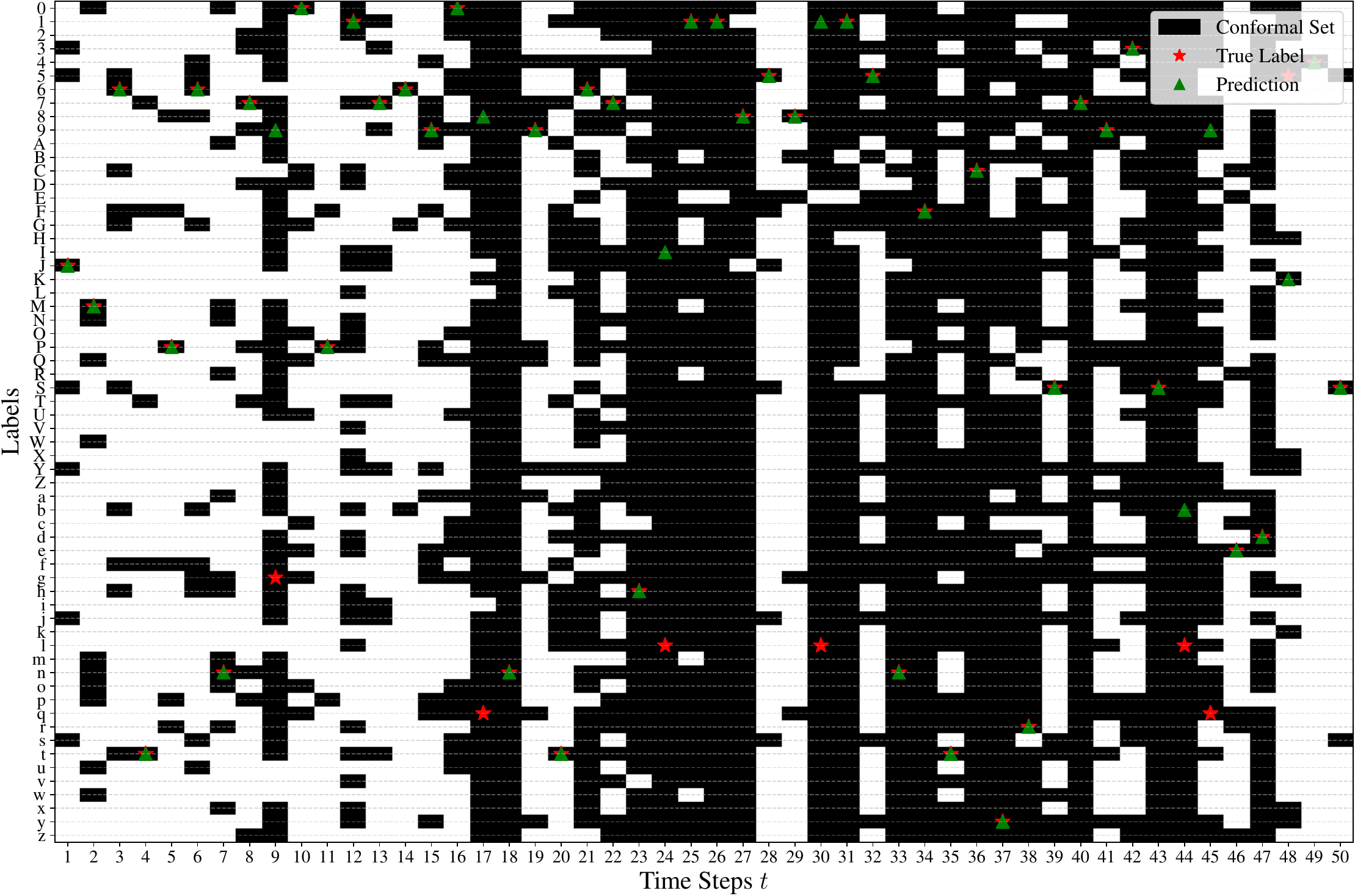}
    \caption{Example 4 of sequence of batch anytime-valid conformal sets.}
    \label{fig:bav-conformal-sequence4}
\end{figure}

\begin{figure}[h!]
    \centering
    \includegraphics[width=.9\textwidth]{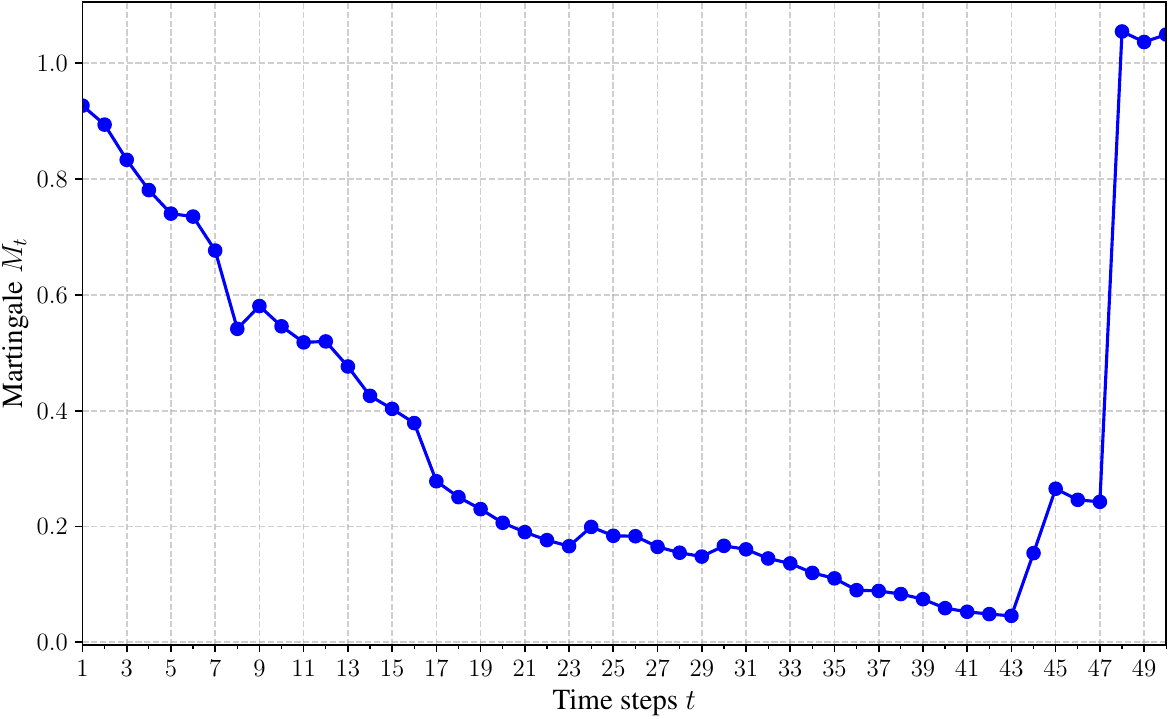}
    \caption{Sample path of the martingale $M_t$ associated with the sequence of batch anytime-valid conformal sets from Example 4.}
    \label{fig:bav-martingale4}
\end{figure}

\begin{figure}[h!]
    \centering
    \includegraphics[width=.9\textwidth]{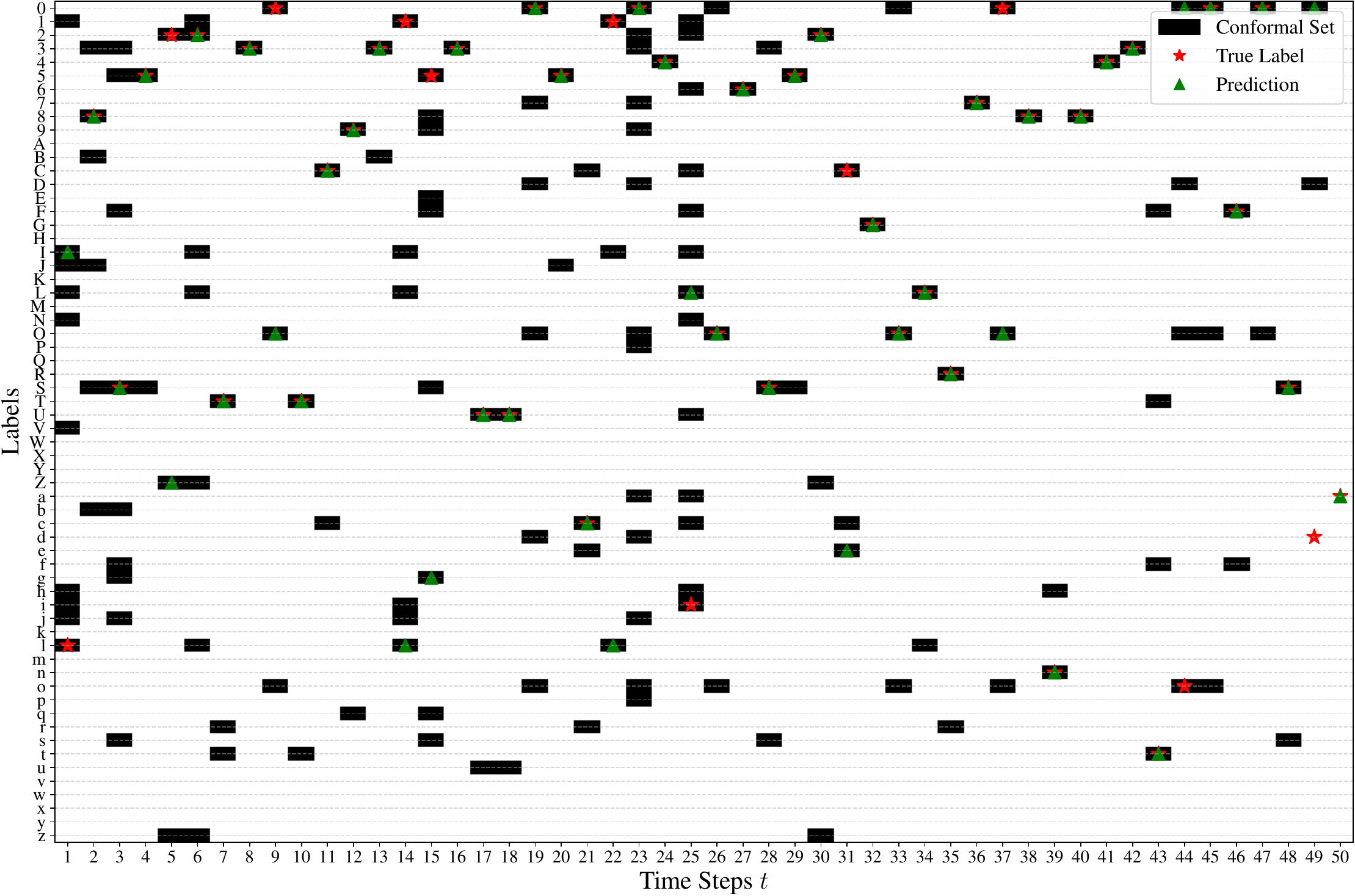}
    \caption{Example 5 of sequence of batch anytime-valid conformal sets.}
    \label{fig:bav-conformal-sequence}
\end{figure}

\begin{figure}[h!]
    \centering
    \includegraphics[width=.9\textwidth]{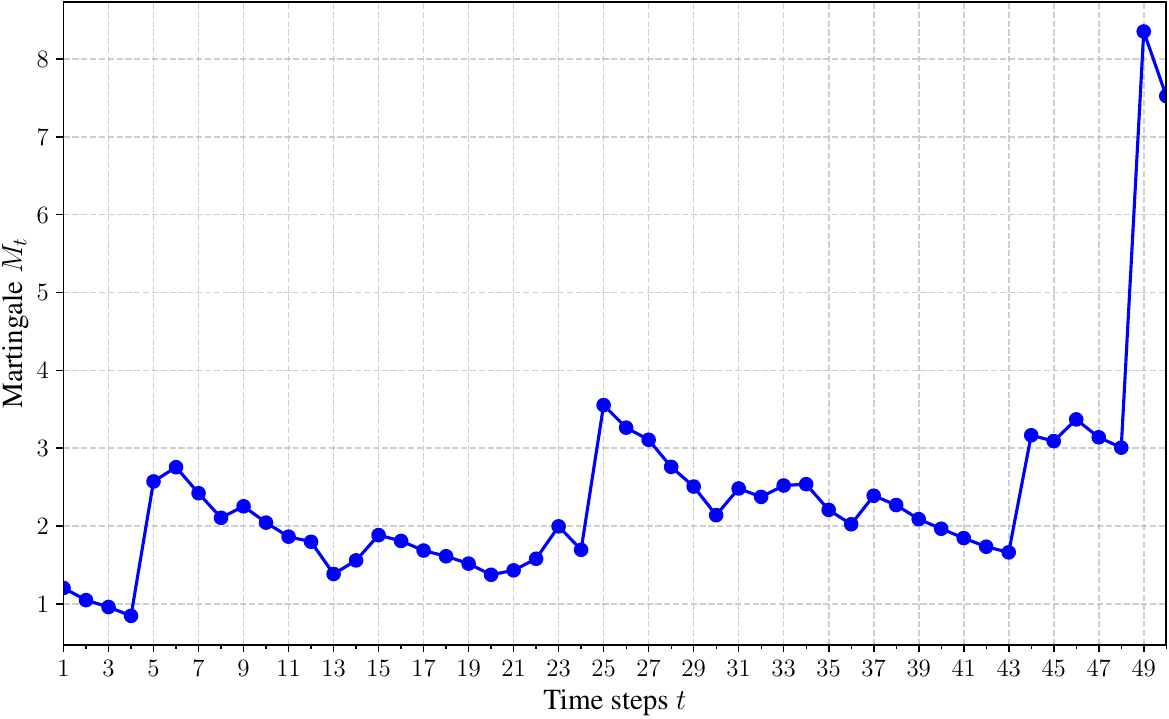}
    \caption{Sample path of the martingale $M_t$ associated with the sequence of batch anytime-valid conformal sets from Example 5.}
    \label{fig:bav-martingale}
\end{figure}

\clearpage
\section{Additional Plots for Section \ref{sec:post-hoc}}
\label{app:post-hoc}

We present additional sample results from the experiments in Section \ref{sec:post-hoc}, in Figure \ref{fig:example5} and Figure \ref{fig:example6}. For example, if we aim to obtain conformal sets with a size (at most) $C=5$, applying Equation (\ref{eq:def-alpha-tilde}) yields $\tilde{\alpha}=0.11$ in Example 2 of Figure \ref{fig:example5}, and $\tilde{\alpha}=0.04$ in Example 3 of Figure \ref{fig:example6}.

\begin{figure}[h!]
    \centering
    \includegraphics[width=\textwidth]{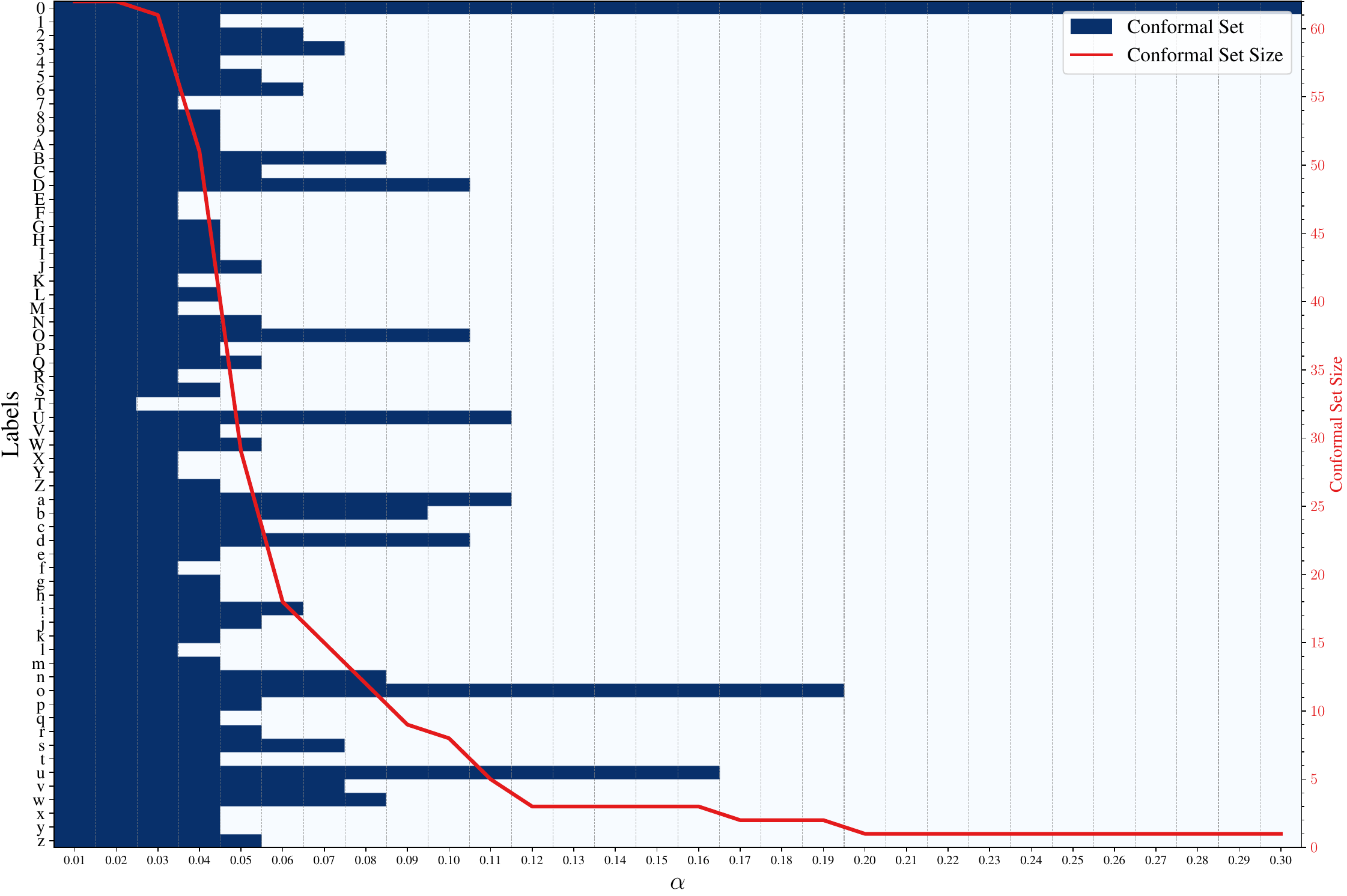}
    \caption{Example 2 of conformal sets obtained with varying $\tilde{\alpha}$.}
    \label{fig:example5}
\end{figure}

\begin{figure}[h!]
    \centering
    \includegraphics[width=\textwidth]{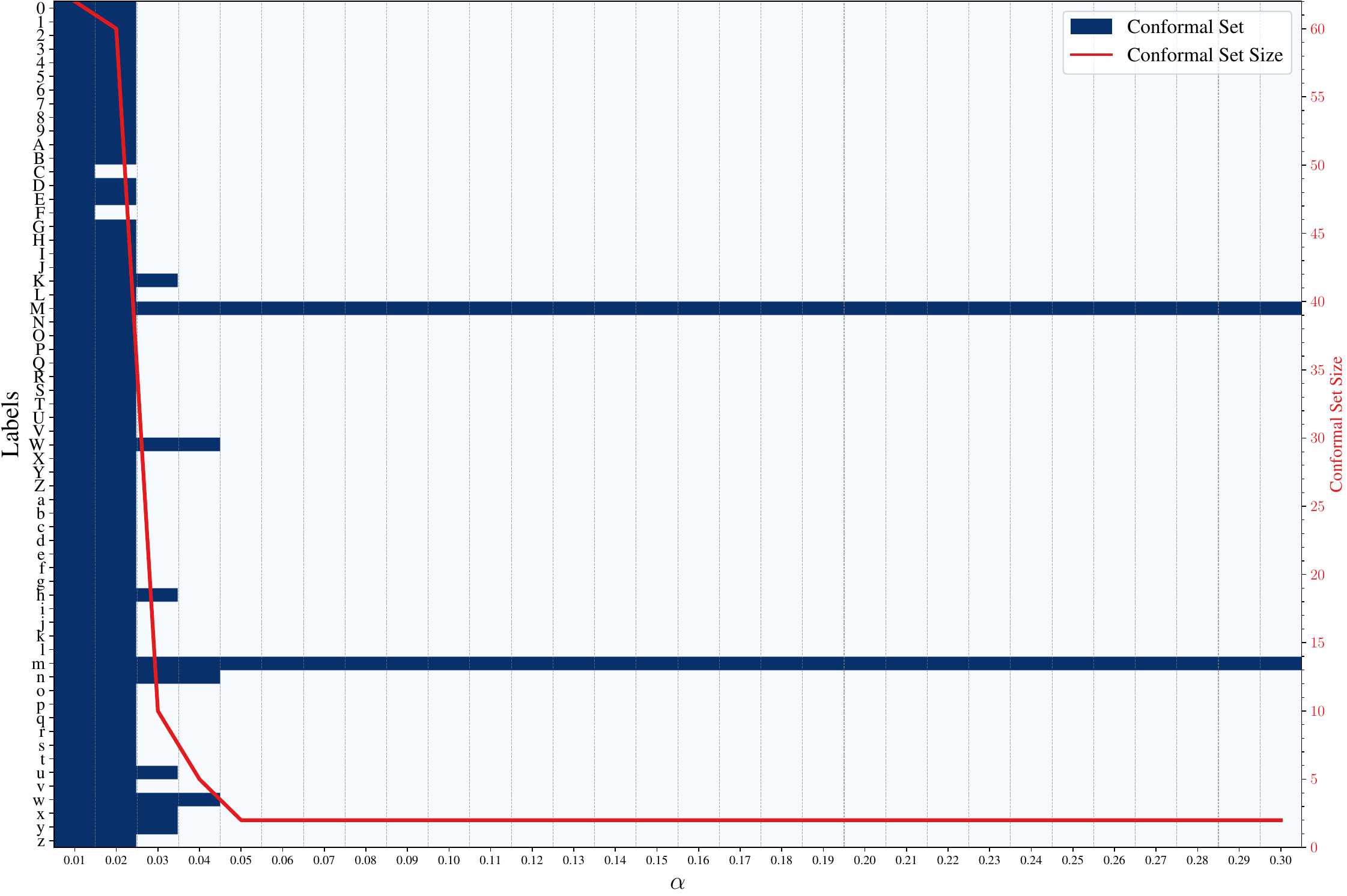}
    \caption{Example 3 of conformal sets obtained with varying $\tilde{\alpha}$.}
    \label{fig:example6}
\end{figure}

\end{document}